%% file: main.tex
\documentclass[11pt]{article}
\usepackage{lmodern}
\PassOptionsToPackage{numbers, sort&compress}{natbib}
\usepackage[utf8]{inputenc} % allow utf-8 input
\usepackage[T1]{fontenc}    % use 8-bit T1 fonts

\RequirePackage[colorlinks,citecolor=blue,urlcolor=blue]{hyperref}

\usepackage{url}            % simple URL typesetting
\usepackage{booktabs}       % professional-quality tables
\usepackage{amsfonts}       % blackboard math symbols
\usepackage{nicefrac}
\usepackage{wrapfig}% compact symbols for 1/2, etc.
\usepackage{microtype}      % microtypography
\usepackage{amsmath, amsthm} % tr`es bon mode math?matique
\usepackage{amsfonts,amssymb}
\usepackage{algorithm,algorithmic}
\usepackage{nicefrac}       % compact symbols for 1/2, etc.
\usepackage{microtype}      % microtypography

\usepackage{xcolor}
\usepackage{graphicx}
\usepackage{wrapfig,lipsum}
\usepackage{thmtools, thm-restate}
\usepackage{mathrsfs} %pour mathscr

\usepackage[shortlabels]{enumitem}

\usepackage{cite}

\usepackage{caption}

\usepackage[caption=false]{subfig}
\usepackage{bbm, dsfont}

\usepackage[capitalize]{cleveref}

\input{macro-arxiv.tex}

\title{\LARGE\bf\spmlTitle{}\vspace{1em}}

\author{ Ruohan Wang$^{1}$ \\ {\footnotesize\em r.wang16@imperial.ac.uk} \and  Yiannis Demiris$^{1}$ \\ {\footnotesize\em y.demiris@imperial.ac.uk} \\ \and  Carlo Ciliberto$^{1}$ \\ {\footnotesize\em c.ciliberto@imperial.ac.uk} \\ $ $ \\  }

\begin{document}

\maketitle

\begin{abstract}
\noindent \footnotetext[1]{Electrial and Electronic Engineering Department, Imperial College London, SW7 2BT, London, United Kingdom.}The goal of optimization-based meta-learning is to find a single initialization shared across a distribution of tasks to speed up the process of learning new tasks.
\textit{Conditional} meta-learning seeks task-specific initialization to better capture complex task distributions and improve performance.
However, many existing conditional methods are difficult to generalize and lack theoretical guarantees.
In this work, we propose a new perspective on conditional meta-learning via structured prediction.
We derive \textit{task-adaptive structured meta-learning} (\spml{}), a principled framework that yields task-specific objective functions by weighing meta-training data on target tasks.
Our non-parametric approach is model-agnostic and can be combined with existing meta-learning methods to achieve conditioning.
Empirically, we show that \spml{} improves the performance of existing meta-learning models, and outperforms the state-of-the-art on benchmark datasets.
\end{abstract}

\section{Introduction}
\label{sec:intro}
State-of-the-art learning algorithms such as neural networks typically require vast amounts of data to generalize well. This is problematic for applications with limited data availability (e.g. drug discovery~\cite{altae2017low}). Meta-learning is often employed to tackle the lack of training data~\cite{finn2017model, vinyals2016matching, ravi2016optimization}. It is designed to learn data-driven inductive bias to speed up learning on novel tasks~\cite{thrun1996learning,vilalta2002perspective}, with many application settings such as learning-to-optimize~\cite{li2016learning} and few-shot learning~\cite{fei2006one, lake2011one}. Meta-learning methods could be broadly categorized into metric-learning~\cite{vinyals2016matching, snell2017prototypical, oreshkin2018tadam}, model-based~\cite{li2016learning, hochreiter2001learning, ha2016hypernetworks}, and optimization-based~\cite{finn2017model, rusu2018meta, nichol2018first}.

We focus on optimization-based approaches, which cast meta-learning as a bi-level optimization~\cite{finn2017model, rajeswaran2019meta, antoniou2018train}. At the single-task level, an ``inner'' algorithm performs task-specific optimization starting from a set of meta-parameters shared across all tasks. At the ``outer'' level, a meta learner accrues experiences from observed tasks to learn the aforementioned meta-parameters. These methods seek to learn a single initialization of meta-parameters that can be effectively adapted to all tasks. Relying on the shared initialization is challenging for complex (e.g. multi-modal) task distributions, since different tasks may require a substantially different initialization, given the same adaptation routine. This makes it infeasible to find a common meta-parameters for all tasks. Several recent works~\cite{vuorio2019multimodal,yao2019hierarchically,rusu2018meta,jerfel2019reconciling,cai2020weighted, jiang2018learning, lee2019learning, wang2019tafe,denevi2020advantage} address the issue by conditioning such parameters on target tasks, and demonstrate consistent improvements over unconditional meta-learning. However, existing methods often lack theoretical guarantees on generalization performance and implements specific conditioning principles with customized networks, which may be difficult to generalize across different application settings.

In this paper, we offer a novel perspective on \textit{conditional} meta-learning based on structured prediction~\cite{bakir2007predicting}. This enables us to propose {\itshape \spmlong{} (\spml)} -- a general framework for conditional meta-learning -- by interpreting the inner algorithm as the structured output to be predicted, conditioned on target tasks. We derive a principled estimator that minimizes task-specific meta-learning objectives, which weigh known training tasks based on their similarities with the target task. The proposed framework is non-parametric and thus requires access to training data at test time for the task-specific objectives. We introduce an efficient algorithm for \spml{} to mitigate the additional computational costs associated with optimizing these task-specific objectives. Intuitively, the proposed framework learns a target task by explicitly recalling only the most relevant tasks from past observations, to better capture the local task distribution for improved generalization. The relevance of previously observed tasks with respect to the target one is measured by a structured prediction approach from \cite{ciliberto2019localized}. \spml{} is model-agnostic and can easily adapt existing meta-learning methods to achieve conditional meta-learning.

We empirically evaluate \spml{} on two competitive few-shot classification benchmarks: {\itshape mini-} and \timg{}. We show that \spml{} outperforms state-of-the-art methods, and improves the accuracy of existing meta-learning algorithms by adapting them into their respective conditional variants. We also investigate \spml{}'s trade-off between computational efficiency and accuracy improvement, showing that the proposed method achieves good efficiency in learning new tasks.

Our main contributions include: $i)$ a new perspective on conditional meta-learning based on structured prediction, $ii)$ \spml{}, a conditional meta-learning framework that generalizes existing meta-learning methods, $iii)$ a practical and efficient algorithm under the proposed framework, and $iv)$ a thorough evaluation of \spml{} on benchmarks, outperforming state-of-the-art methods.

\section{Background and Notation}
\label{sec:bg}
For clarity, in the following we focus on meta-learning for supervised learning. However, the discussion below and our proposed approach also apply to general learning settings. 

% {\bfseries Supervised learning.}
\paragraph{Supervised learning}
In supervised learning, given a probability distribution $\rho$ over two spaces $\X\times\Y$ and a loss $\ell:\Y\times\Y\to\R$ measuring prediction errors, the goal is to find $f:\X\to\Y$ that minimizes the {\itshape expected risk}
\eqal{\label{eq:supervised-risk}
    \min_{f:\X\to\Y}~ \E(f) \quad ~\textrm{with}~ \quad \E(f) ~=~ \EE_\rho~ \ell(f(x),y),
}
with $(x,y)$ sampled from $\rho$. A finite training set $D = (x_j,y_j)_{j=1}^m$ of \textit{i.i.d.} samples from $\rho$ is given. A learning algorithm typically finds $f\in\F$ within a prescribed set $\F$ of candidate models (e.g. neural networks, reproducing kernel Hilbert spaces), by performing empirical risk minimization on $D$ or adopting online strategies such as stochastic gradient methods (SGD) (see \cite{shalev2014understanding} for an in-depth view on statistical learning). A learning algorithm may thus be seen as a function $\alg:\D\to\F$ that maps an input dataset $D$ to a model $f = \alg(D)$, where $\D$ is the space of datasets on $\X\times\Y$.

\paragraph{Meta-learning}
While in supervised settings $\alg(\cdot)$ is chosen a-priori, the goal of meta-learning is to {\itshape learn a learning algorithm} suitable for a family of tasks. Thus, we consider a parametrization $\alg(\theta,\cdot):\D\to\F$ for the inner algorithm, with $\theta\in\Theta$ a space of meta-parameters and aim to solve
\eqal{\label{eq:meta-learning-risk}
\min_{\theta\in\Theta} ~\E(\theta) \qquad \textrm{with} \qquad \E(\theta) ~=~ \EE_\mu \EE_\rho ~\Lagr\big(\alg(\theta,D^{tr}),~D^{val}\big).
}
Here, $\rho$ is a task distribution sampled from a meta-distribution $\mu$, and $D^{tr}$ and $D^{val}$ are respectively training and validation sets of {\itshape i.i.d.} data points $(x,y)$ sampled from $\rho$. The task loss $\Lagr:\F\times\D\to\R$ is usually an empirical average of the prediction errors on a dataset according to an inner loss $\ell$,
\eqal{\label{eq:meta-loss}
    \Lagr(f,D) ~=~ \frac{1}{|D|}\sum_{(x,y)\in D} \ell(f(x),y),
}
with $|D|$ the cardinality of $D$. We seek the best $\theta^*$ such that applying $\alg(\theta^*,\cdot)$ on $D^{tr}$ achieves lowest generalization error on $D^{val}$, among all algorithms parametrized by $\theta\in\Theta$. In practice, we have access to only a finite meta-training set $\metaD = (D^{tr}_i,D^{val}_i)_{i=1}^N$ and the meta-parameters $\hat\theta$ are often learned by (approximately) minimizing 
\eqal{\label{eq:erm-meta-learning}
    \hat\theta = \argmin_{\theta\in\Theta}~\frac{1}{N}\sum_{i=1}^N~ \Lagr\big(\alg(\theta,D_i^{tr}), D_i^{val}\,\big).
}
Meta-learning methods address \cref{eq:erm-meta-learning} via first-order methods such as SGD, which requires access to $\nabla_\theta \alg(\theta,D)$, the (sub)gradient of the inner algorithm over its meta-parameters. For example, model-agnostic meta-learning (MAML)~\cite{finn2017model} and several related methods~(e.g. \cite{antoniou2018train,li2017meta,rajeswaran2019meta}) cast meta-learning as a bi-level optimization problem. In MAML, $\theta$ parametrizes a model $f_\theta:\X\to\Y$ (e.g. a neural network), and $\alg(\theta,D)$ performs one (or more) steps of gradient descent minimizing the empirical risk of $f_\theta$ on $D$. Formally, given a step-size $\eta>0$, 
\eqals{
f_{\theta'} = \alg(\theta,D) \quad ~\textrm{with}~ \quad \theta' = \theta - \eta~ \nabla_\theta \Lagr(f_\theta,D).
}
Inspired by MAML, meta-representation learning~\cite{zintgraf2018fast, bertinetto2018meta} performs task adaptation via gradient descent on only a subset of model parameters and considers the remaining ones as a shared representation.

\section{Conditional Meta-Learning}
\label{sec:conditional-meta-learning}
Although remarkably efficient in practice, optimization-based meta-learning typically seeks a single set of meta-parameters $\theta$ for all tasks from $\mu$. This shared initialization might be limiting for complex (e.g. multi-modal) task distributions: dissimilar tasks require substantially different initial parameters given the same task adaptation routine, making it infeasible to find a common initialization~\cite{yao2019hierarchically, vuorio2019multimodal}. To address this issue, several recent works learn to condition the initial parameters on target tasks (see \cref{fig:motivation} in \cref{sec:practical} for a pictorial illustration of this idea). For instance, \cite{vuorio2019multimodal, yao2019hierarchically, rusu2018meta} directly learn data-driven mappings from target tasks to initial parameters, and \cite{jiang2018learning} conditionally transforms feature representations based on a metric space trained to capture inter-class dependencies. Alternatively, \cite{jerfel2019reconciling} considers a mixture of hierarchical Bayesian models over the parameters of meta-learning models to condition on target tasks, while \cite{cai2020weighted} preliminarily explores task-specific initialization by optimizing weighted objective functions. However, these existing methods typically implement specific conditional principles with customized network designs, which may be difficult to generalize to different application settings. Further, they often lack theoretical guarantees.

\paragraph{Conditional Meta-learning} We formalize the conditional approaches described above as \textit{conditional meta-learning}. Specifically, we condition the meta-parameters $\theta$ on $D$ by parameterizing $\alg(\condmetapar(D),\cdot)$ with $\tau(D)\in\Theta$, a meta-parameter valued function. We cast {\itshape conditional meta-learning} as a generalization of \cref{eq:meta-learning-risk} to minimize
\eqal{\label{eq:conditional-meta-learning-risk}
\min_{\condmetapar:\D\to\Theta}~\E(\condmetapar) \qquad \textrm{with} \qquad \E(\condmetapar) = \EE_{\mu}\EE_\rho ~\Lagr\Big(\alg\big(\condmetapar(D^{tr}),~D^{tr}\big),~D^{val}\Big),
}
over a suitable space of functions $\condmetapar:\D\to\Theta$ mapping datasets $D$ to algorithms $\alg(\condmetapar(D),\cdot)$. While \cref{eq:conditional-meta-learning-risk} uses $D^{tr}$ for both the conditioning and inner algorithm, more broadly $\condmetapar$ may depend on a separate dataset $\tau(D^{con})$ of ``contextual'' information (as recently investigated also in \cite{denevi2020advantage}), similar to settings like collaborative filtering with side-information~\cite{abernethy2009new}. Note that standard (unconditional) meta-learning can be interpreted as an instance of \cref{eq:conditional-meta-learning-risk} with $\tau(D)\equiv\theta$, the constant function associating every dataset $D$ to the same meta-parameters $\theta$. Intuitively, we can expect a significant improvement from the solution $\condmetapar_*$ of \cref{eq:conditional-meta-learning-risk} compared to the solution $\theta_*$ of \cref{eq:meta-learning-risk}, since by construction $\E(\tau_*)\leq\E(\theta_*)$.

Conditional meta-learning leverages a finite number of meta-training task to learn $\tau:\D\to\Theta$. While it is possible to address this problem in a standard supervised setting, we stress that meta-learning poses unique challenges from both modeling and computational perspectives. A critical difference is the output set: in standard settings, this is usually a linear space (namely $\Y = \R^k$), for which there exist several methods to parameterize suitable spaces of hypotheses $f:\X\to\R^k$. In contrast, when the output space $\Theta$ is a complicated, ``structured'' set (e.g. space of deep learning architectures), it is less clear how to find a space of hypotheses $\condmetapar:\D\to\Theta$ and how to perform optimization over them. These settings however are precisely what the literature of {\itshape structured prediction} aims to address.

\subsection{Structured Prediction for Meta-learning}
Structured prediction methods are designed for learning problems where the output set is not a linear space but rather a set of structured objects such as strings, rankings, graphs, 3D structures \cite{bakir2007predicting,nowozin2011structured}. For conditional meta-learning, the output space is a set of inner algorithms parameterized by $\theta\in\Theta$. Directly modeling $\condmetapar:\D\to\Theta$ can be challenging. A well-established strategy in structured prediction is therefore to first learn a joint function $\jointAlg:\Theta\times\D\to\R$ that, in our setting, measures the quality of a model $\theta$ for a specific dataset $D$. The structured prediction estimator $\condmetapar$ is thus defined as the function choosing the optimal model parameters $\tau(D)\in\Theta$ given the input dataset $D$
\eqal{\label{eq:general-structured-prediction}
    \condmetapar(D) = \argmin_{\theta\in\Theta}~\jointAlg(\theta,D).
}
Within the structured prediction literature, several strategies have been proposed to model and learn the joint function $T$, such as SVMStruct~\cite{tsochantaridis2005} and Maximum Margin Markov Networks~\cite{taskar2004max}. However, most methods have been designed to deal with output spaces $\Theta$ that are discrete or finite and are therefore not suited for conditional meta-learning. To our knowledge, the only structured prediction framework capable of dealing with general output spaces (e.g. dense set $\Theta$ of network parameters) is the recent work based on the {\itshape structure encoding loss function} principle \cite{ciliberto2019localized,ciliberto2020general}. This approach also enjoys strong theoretical guarantees including consistency and learning rates. We propose to apply such a method to conditional meta-learning and then study its generalization properties.

\paragraph{Task-adaptive Structured Meta-Learning}
To apply \cite{ciliberto2019localized}, we assume access to a reproducing kernel \cite{aronszajn1950theory} $k:\D\times\D\to\R$ on the space of datasets (see \cref{eq:dataset-signature} in \cref{sec:implementation} for an example). Given a meta-training set $\metaD = (D^{tr}_i,D^{val}_i)_{i=1}^N$ and a new task $D$, the structured prediction estimator is
\eqal{\label{eq:estimator}
\begin{split}
& \tau(D) = \argmin_{\theta \in \Theta}~ \sum_{i=1}^N ~\alpha_i(D)~\Lagr\big(\alg(\theta,D_i^{tr}),~D_i^{val}\,\big)\\
& ~~\textrm{with} \quad~~ \alpha(D) = (\mbf{K}+\lambda I)^{-1}v(D) \in \mathbb{R}^N,
\end{split}
}
where $\lambda > 0$ is a regularizer, $\alpha_i(D)$ denotes the $i$-th entry of the vector $\alpha(D)$ while $\mbf{K}\in \mathbb{R}^{N\times N}$ and $v(D)\in\R^N$ are the kernel matrix and evaluation vector with entries $\mbf{K}_{i,j} = k(D^{tr}_i, D^{tr}_j)$ and $v(D)_{i} = k(D^{tr}_i, D)$, respectively. We note that \cref{eq:estimator} is an instance of \cref{eq:general-structured-prediction}, where the joint functional $\jointAlg$ is modelled according to \cite{ciliberto2019localized} and learned on the meta-training set $\metaD$. The resulting approach is a non-parametric method for conditional meta-learning, which accesses training tasks for solving \cref{eq:estimator}.

We refer to the estimator in \cref{eq:estimator} as {\itshape \spmlong{} (\spml{})}. In this formulation, we seek $\theta$ to minimize a weighted meta-learning objective, where the $\alpha:\D\to\R^N$ can be interpreted as a ``scoring'' function that measures the relevance of known training tasks to target tasks. The structured prediction process is hence divided into two distinct phases: $i)$ a {\itshape learning} phase for estimating the scoring function $\alpha$ and $ii)$ a {\itshape prediction} phase to obtain $\condmetapar(D)$ by solving \cref{eq:estimator} on $D$. The following remark draws a connection between \spml{} and unconditional meta-learning methods. 
\begin{remark}[Connection with standard meta-learning]\label{rem:connection-with-maml}
The objective in \cref{eq:estimator} recovers the empirical risk minimization for meta-learning introduced in \cref{eq:erm-meta-learning} if we set $\alpha_i(D) \equiv 1$. Hence, methods from \cref{sec:bg} -- such as MAML -- can be interpreted as conditional meta-learning algorithms that assume all tasks being equally related to one other.
\end{remark}
\cref{rem:connection-with-maml} suggests that \spml{} is compatible with most existing (unconditional) meta-learning methods in the form of \cref{eq:erm-meta-learning} (or its stochastic variants). Thus, \spml{} can leverage existing algorithms, including their architectures and optimization routines, to solve the weighted meta-learning problem in \cref{eq:estimator}. In \cref{sec:exp}, we show empirically that \spml{} improves the generalization performance of three meta-learning algorithms by adopting the proposed structured prediction perspective. Additionally, in \cref{sec:exp_mtl} we discuss the potential relations between \cref{eq:estimator} and recent multi-task learning (MTL) strategies that rely on task-weighing \cite{chen2018gradnorm, kendall2018multi}.

\paragraph{Theoretical Properties}
Thanks to adopting a structured prediction perspective, we can characterize \spml{}'s learning properties. In particluar, the following result provides non-asymptotic excess risk bounds for our estimator that indicate how fast we can expect the prediction error of $\condmetapar$ to decrease as the number $N$ of meta-training tasks grows.

\begin{restatable}[{\itshape Informal} -- Learning Rates for \spml{}]{theorem}{TRatesInformal}\label{thm:rates-informal} Let $\metaD = (D^{tr}_i,D^{val}_i)_{i=1}^N$ be sampled from a meta-distribution $\mu$ and  $\condmetapar_N$ the estimator in \cref{eq:estimator} trained with $\lambda = N^{-1/2}$  on $\metaD$. Then, with high probability with respect to $\mu$,
\eqal{
    \E(\condmetapar_N) ~- \inf_{\condmetapar:\D\to\Theta}~\E(\condmetapar)~\leq O(N^{-1/4}).
}
\end{restatable}

\cref{thm:rates-informal} shows that the proposed algorithm asymptotically yields the \textit{best} possible task-conditional estimator for the family of tasks identified by $\mu$, over the novel samples from validation sets. The proof of \cref{thm:rates-informal} leverages recent results from the literature on structured prediction \cite{luise2018differential,rudi2018manifold}, combined with standard regularity assumptions on the meta-distribution $\mu$. See \cref{app:theory} for a proof and further discussion on the relation between \spml{} and general structured prediction.

\section{A Practical Algorithm for \spml{}}
\label{sec:practical}

The proposed \spml{} estimator $\tau:\D\to\Theta$ offers a principled approach to conditional meta-learning. However, task-specific objectives incur additional computational cost compared to unconditional meta-learning models like MAML, since we have to repeatedly solve \cref{eq:estimator} for each target task $D$, in particular when the number $N$ of meta-training tasks is large. We hence introduce several adjustments to \spml{} that yield a significant speed-up in practice, without sacrificing overall accuracy.

\paragraph{Initialization by Meta-Learning}
Following the observation in \cref{rem:connection-with-maml}, we propose to learn an ``agnostic'' $\hat\theta\in\Theta$ as model initialization before applying \spml{}. Specifically, we obtain $\hat\theta$ by applying a standard (unconditional) meta-learning method solving \cref{eq:erm-meta-learning}. We then initialize the inner algorithm with $\hat\theta$, followed by minimizing \cref{eq:estimator} over the meta-parameters. In practice, the learned initialization significantly speeds up the convergence in \cref{eq:estimator}. We stress that the proposed initialization is optional: directly applying \spml{} to each task using random initialization obtains similar performance, although it takes more training iterations to achieve convergence (see \cref{app:no-init}).

\paragraph{Top-$M$ Filtering}
In \cref{eq:estimator}, the weights $\alpha_i(D)$ measure the relevance of the $i$-th meta-training task $D_i^{tr}$ to the target task $D$. We propose to keep only the top-$M$ values from $\alpha(D)$, with $M$ a hyperparameter, and set the others to zero. This filtering reduces the computational cost of \cref{eq:estimator}, by constraining training to only tasks $D_i^{tr}$ most relevant to $D$, The filtering process has little impacts on the final performance, since we empirically observed that only a small percentage of tasks have large $\alpha(D)$, given a large number $N$ of training tasks (e.g. $N > 10^5$). In our experiments we chose $M$ to be $\sim1$\% of $N$, since larger values did not provide significant improvements in accuracy (see \cref{app:sp_top_m} for further ablation). We observe that \spml{}'s explicit dependence on meta-training tasks resembles meta-learning methods with external memory module~\cite{santoro2016meta}, The proposed filtering process in turn resembles memory access rules but requires no learning. For each target task, only a small number of training tasks are accessed for adaptation, limiting the overall memory requirements.

\paragraph{Task Adaptation} 
The output $\tau(D)$ in \cref{eq:estimator} depends on the target task $D$ only via the task weights $\alpha(D)$. We propose to directly optimize $D$ with an additional term $\Lagr(\alg(\theta, D), D)$, which encourages $\theta$ to directly exploit the training signals in target task $D$ and achieve small empirical error on it. The extra term may be interpreted as adding a ``special'' training task $(\tilde D^{tr}, \tilde D^{val}) = (D, D)$, in which the support set and query set coincides, to the meta-training set $(D^{tr}_i, D^{val}_i)_{i=1}^N$
\eqal{\label{eq:estimator-improved}
    \condmetapar(D) ~=~ \argmin_{\theta \in \Theta} ~ \sum_{i=1}^N \alpha_i(&D) ~ \Lagr\big(\alg(\theta, D^{tr}_i),~ D^{val}_i\big) ~+~ \Lagr\big(\alg(\theta, D),~ D\big),
}
By construction, the additional task offers useful training signals by regularizing the model to focus on relevant features shared among the target task and past experiences. We refer to \cref{app:estimator_ablation} for an ablation study on how \cref{eq:estimator-improved} affects test performance.

% While in principle this strategy might be prone to overfitting, the first term in \cref{eq:estimator-improved} prevents this to happen by optimizing with respect to the meta-training set.

\begin{figure}[t]
\noindent\begin{minipage}{.49\textwidth}
\begin{algorithm}[H]
   \caption{\spml{}\label{alg:spml}}
\begin{algorithmic}
    \STATE \hspace{-1em}{\bfseries Input:} meta-train set $\metaD = (D^{tr}_i, D^{val}_i)_{i=1}^N$,\\  \hspace{-1em}initial parameters $\theta$, Filter size $M$, step-size $\eta$,\\ \hspace{-1em}kernel $k$, regularizer $\lambda$ 
    
    \vspace{0.5em}
    \hspace{-1em}{\textbf{Meta-Train}}:
    
    \STATE Compute the kernel matrix $\mbf{K}\in\R^{N\times N}$ on $S$
    \STATE Let $v:\D\to\R^N$ where $v(\cdot)_{i} = k(D^{tr}_i, \cdot)$
    \STATE Let $\alpha(\cdot) = (\mbf{K}+\lambda I)^{-1}v(\cdot)$.
    % \STATE Learn $\alpha:\D\to\R^N$ in \cref{eq:estimator} from $\mbf{K}$. 
    \STATE Initialize $\theta$ with \cref{eq:erm-meta-learning} (Optional)
    
    \vspace{0.5em}
    \hspace{-1em}{\textbf{Meta-Test} with target task $D$:}
    \STATE Compute task weights $\alpha(D)\in\R^N$.
    \STATE Get top-$M$ tasks $\metaD_M\subset\metaD$ with highest $\alpha(D)$

    \vspace{0.5em}
    \WHILE{not converged}
        \STATE Sample a mini-batch $\metaD_B \subset \metaD_M$ i.i.d.
        \STATE Compute gradient $\nabla_\theta$ of \cref{eq:estimator-improved} over $\metaD_B$.
        \STATE $\theta \gets \theta - \eta ~\nabla_\theta$
    \ENDWHILE
    
    \vspace{0.5em}
\STATE \hspace{-1em}{\bfseries Return} $\theta$
\end{algorithmic}
\end{algorithm}
\end{minipage}%
\hfill
\begin{minipage}[c]{.49\textwidth}
    \centering
    \includegraphics[width=0.99\columnwidth,trim={0.2cm 0.1cm 0.1cm 0.1cm},clip]{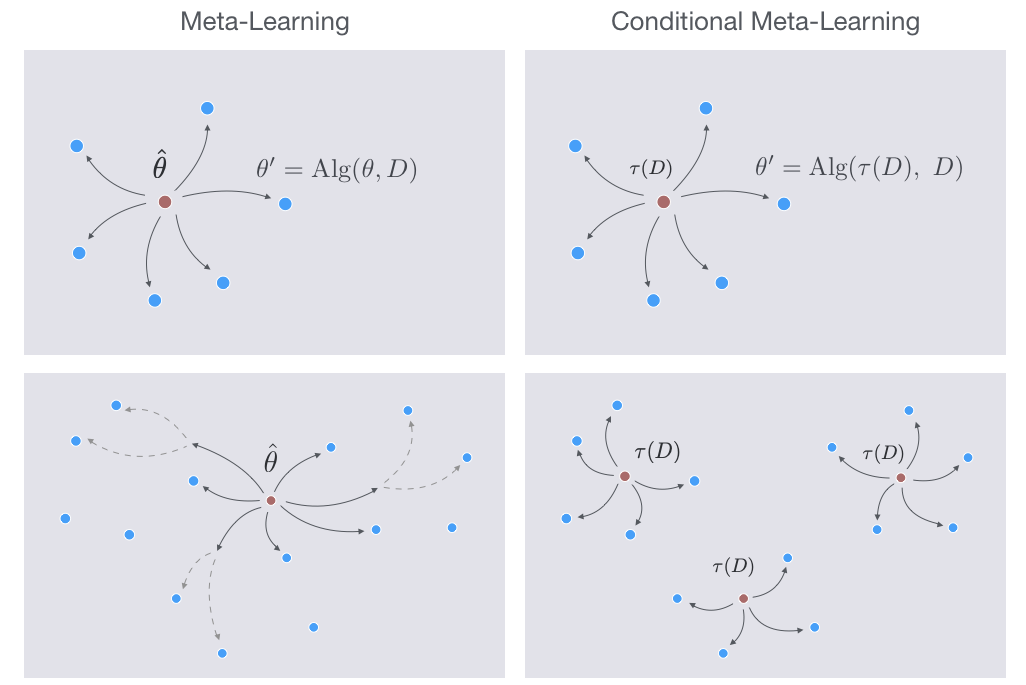}
  \captionof{figure}{Comparing unconditioned (left) and conditional (right) meta-learning. Unconditional methods using a shared initialization (red dot) may fail to adapt to dissimilar tasks (outer blue dots, bottom left). Conditional meta-learning can handle such setting via adaptive initialization (red dots, bottom right).}\label{fig:motivation}
\end{minipage}
\end{figure}

\cref{alg:spml} implements \spml{} with the proposed improvements. We initialize the model with meta-parameters $\theta$, by solving the (unconditional) meta-learning problem in \cref{eq:erm-meta-learning} over a meta-training set $\metaD = (D^{tr}_i, D^{val}_i)_{i=1}^N$. We also learn the scoring function $\alpha$ according to \cref{eq:estimator} by inverting the kernel matrix $\mbf{K}$. While this is an expensive step of up to $O(N^3)$ in complexity, sketching methods may be used to significantly speed up the inversion without loss of accuracy~\cite{rudi2017falkon,meanti2020kernel}. We stress that the scoring function $\alpha$ is learned only once for all target tasks. For any target task $D$, we compute $\alpha(D)$ and only keep the top-$M$ tasks $\metaD_M\subset\metaD$ with largest $\alpha_i(D)$. Lastly, we minimize \cref{eq:estimator-improved} over $\metaD_M$.

\subsection{Implementation Details}
\label{sec:implementation}
\paragraph{Reproducing Kernel on Datasets} 
\spml{} requires a positive definite kernel $k:\D\times\D\to\R$ to learn the score function $\alpha:\D\to\R^N$ in \cref{eq:estimator}. In this work, we take $k$ as the Gaussian kernel of the {\itshape maximum mean discrepancy (MMD)}~\cite{gretton2012kernel} of two datasets. MMD is a popular distance metric on datasets or distributions. More precisely, given two datasets $D,D' \in\D$ and a feature map $\varphi:\X\to\R^p$ on input data, we consider the kernel $k:\D\times\D\to\R$ defined as
\eqal{\label{eq:dataset-signature}
    k(D,D') = \exp\Big(-\nor{\bar\varphi(D)-\bar\varphi(D')}^2/\sigma^2\Big) \qquad \textrm{with} \qquad \bar\varphi(D)=\frac{1}{|D|}\sum_{(x,y)\in D} \varphi(x)
}
% \vskip -0.3cm
where $\sigma>0$ is a bandwidth parameter and $\bar\varphi(D)$ is the {\itshape kernel mean embedding} (or signature) of a dataset $D$ with respect to $\varphi$.

The map $\varphi$ plays a central role. It can be either fixed a priori or learned, depending on the application. A good practice when working with MMD is to use a characteristic kernel \cite{sriperumbudur2010hilbert}. In our experiments, the Gaussian kernel in \cref{eq:dataset-signature} led to best performance (see \cref{app:sp_kernel} for other options). We note that, more generally, $\varphi$ could also be taken to be a joint feature map on both input and output. We describe our choice of feature map $\varphi$ below.

\paragraph{Pre-trained Feature Map}
Expressive input representation plays a significant role in meta-learning, and may be obtained via standard supervised learning (see e.g. \cite{rodriguez2020embedding}). We choose the pre-trained feature map from \cite{rusu2018meta} for the kernel mean embedding $\bar\varphi(D)$ in \cref{eq:dataset-signature} and input representation $\varphi(x)$. 

\paragraph{Model Architecture} 
As observed in \cref{rem:connection-with-maml}, \spml{} is compatible with a wide range of existing meta-learning algorithms. In \cref{sec:sp_comp}, we report on several implementations of \cref{alg:spml}, leveraging architectures and optimization techniques from existing meta-learning methods. In addition, we introduce \lsmetal{}, a least-squares meta-learning algorithm that is highly effective in combination with \spml{}. Similar to \cite{bertinetto2018meta}, we choose an ``inner'' algorithm that solves a least-squares objective $\ell(y,y') = \nor{y-y'}^2$ in closed-form. We propose to induce the {\itshape task loss} $\Lagr$ in \cref{eq:meta-loss} by the same least-squares objective. We note that while least-squares minimization is not a standard approach in classification settings, it is theoretically principled (see e.g. \cite{bartlett2006,mroueh2012multiclass}) and provides a significant improvement to classification accuracy. See \cref{app:ls-meta} for more details.

\section{Experiments}
\label{sec:exp}
We perform experiments\footnote{\spml{} implementation is available at \url{https://github.com/RuohanW/Tasml}} on \textit{C}-way-\textit{K}-shot learning within the episodic formulation of \cite{vinyals2016matching}. In this setting, train-validation pairs $(D^{tr},D^{val})$ are sampled as described in \cref{sec:bg}. $D^{tr}$ is a $C$-class classification problem with $K$ examples per class. $D^{val}$ contains samples from the same $C$ classes for estimating model generalization and training meta-learner. We evaluate the proposed method against a wide range of meta-learning algorithms on three few-shot learning benchmarks: the \mimg{}, \timg{} and \cifar{} datasets. We consider the commonly used $5$-way-$1$-shot and $5$-way-$5$-shot settings. For training, validation and testing, we sample three separate meta-datasets $S^{tr},S^{val}$ and $S^{ts}$, each accessing a disjoint set of classes (e.g. no class in $S^{ts}$ appears in $S^{tr}$ or $S^{val}$). To ensure fair comparison, we adopt the same training and evaluation setup as \cite{rusu2018meta}. \cref{app:details-and-experiments} reports further experimental details including network specification and hyperparameter choice.

\subsection{Experiments on \imgnet{} derivatives}
\label{sec:performance}

We compared \spml{} with a representative set of baselines on classification accuracy. Unconditional methods include \textsc{MAML}~\cite{finn2017model}, \textsc{iMAML}~\cite{rajeswaran2019meta}, \textsc{Reptile}~\cite{nichol2018first}, \textsc{R2D2}~\cite{bertinetto2018meta}, \textsc{(Qiao et al. 2018)}~\cite{qiao2018few}, \textsc{CAVIA}~\cite{zintgraf2018fast}, and  \textsc{META-SGD}~\cite{li2017meta} (using \cite{rusu2018meta}'s features). Conditional methods include \textsc{(Jerfelet al 2019)}~\cite{jerfel2019reconciling}, \textsc{HSML}~\cite{yao2007early}, \textsc{MMAML}~\cite{vuorio2019multimodal}, \textsc{CAML}~\cite{jiang2018learning} and \textsc{\leo{}}~\cite{rusu2018meta}. 

We include results from our local run of \textsc{\leo{}} using the official implementation. In our experiments, we observed that \textsc{\leo{}} appeared sensitive to hyperparameter choices, and obtaining the original performance in \cite{rusu2018meta} was beyond our computational budget. For \textsc{MMAML}, we used the official implementation, since \cite{vuorio2019multimodal} did not report performance on \mimg{}. We did not compare with $\alpha$-MAML~\cite{cai2020weighted} since we did not find ImageNet results nor an official implementation. Other results are cited directly from their respective papers.

\begin{table}[t]
\caption{Classification Accuracy of meta-learning models on \mimg{} and \timg{}.}
\begin{center}
% \begin{small}
\begin{footnotesize}
\begin{sc}
\begin{tabular}{lcc|cc}
\toprule
  & \multicolumn{4}{c}{Accuracy (\%)} \\
  & \multicolumn{2}{c}{\mimg{}} & \multicolumn{2}{c}{\timg{}}\\
Unconditional Methods  & $1$-shot & $5$-shot & $1$-shot & $5$-shot\\
\midrule
MAML \cite{finn2017model} & $48.70 \pm 1.84$ &  $63.11 \pm 0.92$ & $51.67 \pm 1.81$ &  $70.30 \pm 0.8$\\
iMAML \cite{rajeswaran2019meta}& $49.30 \pm 1.88$ & - & - & -\\
Reptile \cite{nichol2018first} & $49.97 \pm 0.32$  & $65.99 \pm 0.58$ & - & -\\
R2D2 \cite{bertinetto2018meta} & $51.90 \pm 0.20$ & $68.70 \pm 0.20$ & - & -\\
CAVIA \cite{zintgraf2018fast} & $51.82 \pm 0.65$ & $65.85 \pm 0.55$ & - & -\\
(Qiao et al.) \cite{qiao2018few}  & $59.60 \pm 0.41$ & $73.74 \pm 0.19$ & - & - \\
Meta-SGD \cite{li2017meta}(\leo{} feat.)& $54.24 \pm 0.03$ & $70.86 \pm 0.04$ & $62.95 \pm 0.03$ &  $79.34 \pm 0.06$ \\
\midrule
Conditional Methods \\
\midrule
(Jerfel et al.) \cite{jerfel2019reconciling} & $51.46 \pm 1.68$ & $65.00 \pm 0.96$ & - & -\\
HSML \cite{yao2019hierarchically} & $50.38 \pm 1.85$ & - & - & -\\
MMAML \cite{vuorio2019multimodal} & $46.1 \pm 1.63$ & $59.8 \pm 1.82$ & - & -\\
CAML \cite{jiang2018learning} & $59.23 \pm 0.99$ & $72.35 \pm 0.71$ & - & -\\
\leo{}~\cite{rusu2018meta} & $61.76 \pm 0.08$ & $77.59 \pm 0.12$ & $66.33 \pm 0.05$ & $81.44 \pm 0.09$\\
\leo{} (local)~\cite{rusu2018meta} & $60.37 \pm 0.74$ & $75.36  \pm 0.44$ & $65.11 \pm 0.72$ & $79.70 \pm 0.59$ \\
\spml{} (Ours) & $\mbf{62.04 \pm 0.52}$ & $\mbf{78.22 \pm 0.47}$  & $\mbf{66.42 \pm 0.37}$ & $\mbf{82.62 \pm 0.31}$\\
\bottomrule
\end{tabular}
\end{sc}
\end{footnotesize}
% \end{small}
\end{center}
%\vskip -0.1in
\label{tab:comp}
\end{table}

\cref{tab:comp} reports the mean accuracy and standard deviation of all the methods over $50$ runs, with each run containing $200$ random test tasks. \spml{} outperforms the baselines in three out of the four settings, and achieves performance comparable to \textsc{\leo{}} in the remaining setting. We highlight the comparison between \spml{} and \textsc{\leo{}} (local), as they share the identical experiment setups. The identical setups make it easy to attribute any relative performance gains to the proposed framework. We observe that \spml{} outperforms \textsc{\leo{}} (local) in all four settings, averaging over 2\% improvements in classification accuracy. The results suggest the efficacy of the proposed method.

\subsection{Experiments on \cifar{}}
The recently proposed \cifar{} dataset~\cite{bertinetto2018meta} is a new few-shot learning benchmark, consisting of all 100 classes from CIFAR-100~\cite{cifar100}. The classes are randomly divided into 64, 16 and 20 for meta-train, meta-validation, and meta-test respectively. Each class includes 600 images of size $32\times 32$.

\begin{table}[t]
\caption{Classification Accuracy of meta-learning models on \cifar{}.}
\begin{center}
% \begin{small}
\begin{footnotesize}
\begin{sc}
\begin{tabular}{lcc}
\toprule
  & $1$-shot & $5$-shot\\
\midrule
MAML~\cite{finn2017model} & $ 58.9 \pm 1.9$ &  $71.5 \pm 1.0$\\
R2D2~\cite{bertinetto2018meta} & $65.3 \pm 0.2$ & $ 79.4 \pm 0.1$\\
ProtoNet(ResNet12 Feat.)~\cite{snell2017prototypical} & $ 72.2 \pm 0.7$ & $83.5 \pm 0.5$\\
MetaOptNet~\cite{lee2019meta} & $72.0 \pm 0.7$ & $84.2 \pm 0.5$  \\
\spml{} & $\mathbf{74.6 \pm 0.7}$ & $\mathbf{85.1 \pm 0.4}$ \\
\bottomrule
\end{tabular}
\end{sc}
\end{footnotesize}
% \end{small}
\end{center}
%\vskip -0.1in
\label{tab:comp_cifar}
\end{table}

\cref{tab:comp_cifar} compares \spml{} to a diverse set of previous methods, including \textsc{MAML}, \textsc{R2D2},  \textsc{ProtoNets}~\cite{snell2017prototypical} and \textsc{MetaOptNets}~\cite{lee2019meta}. The results clearly show that \spml{} outperforms previous methods in both settings on the \cifar{} dataset, further validating the efficacy of the proposed structured prediction approach.

\subsection{Improvements from Structured Prediction}
\label{sec:sp_comp}
The task-specific objective in \cref{eq:estimator-improved} is model-agnostic, which enables us to leverage existing meta-learning methods, including architecture and optimization routines, to implement \cref{alg:spml}. For instance, we may replace \lsmetal{} in \cref{sec:implementation} with \textsc{MAML}, leading to a new instance of structured prediction-based conditional meta-learning. 
% We refer to the combination of structured prediction with each method as \textsc{SP + Model}

\begin{table}[t]
\caption{Effects of structured prediction on \mimg{} and \timg{} benchmarks. Structured prediction (SP) improves the underlying meta-learning algorithms in all cases.}
\begin{center}
% \begin{small}
\begin{footnotesize}
\begin{sc}
\setlength{\tabcolsep}{4pt}
\begin{tabular}{lcccc}
\toprule
  & \multicolumn{4}{c}{Accuracy (\%)} \\
  & \multicolumn{2}{c}{\mimg{}} & \multicolumn{2}{c}{\timg{}}\\
  & $1$-shot & $5$-shot & $1$-shot & $5$-shot\\
\midrule
MAML (LEO Feat.) & $54.12 \pm 1.84$ &  $67.58 \pm 0.92$ & $51.28 \pm 1.81$ &  $69.80 \pm 0.84$\\
SP+MAML (LEO Feat.) & $\mbf{58.46 \pm 1.56}$ &  $\mbf{74.51 \pm 0.75}$ & $\mbf{60.89 \pm 1.64}$ &  $\mbf{78.42 \pm 0.73}$\\
\midrule
\leo{} (local) & $60.37 \pm 0.74$ & $75.36  \pm 0.44$ & $65.11 \pm 0.72$ & $79.70 \pm 0.59$ \\
SP+\leo{} (local) & $\mbf{61.46 \pm 0.69}$ & $\mbf{76.54  \pm 0.59}$ & $\mbf{66.07 \pm 0.66}$ & $\mbf{80.68 \pm 0.41}$ \\
\midrule
% \preTrainFeat & $51.37 \pm 0.39$ &  $69.91 \pm 0.21$ & $57.23 \pm 0.35$ & $78.48 \pm 0.27$\\
\lsmetal{} & $60.19 \pm 0.65$ & $76.76  \pm 0.43$ & $64.32 \pm 0.65$ & $81.43 \pm 0.55$\\
\spml{} (SP + LS Meta-Learn) & $\mbf{62.04 \pm 0.52}$ & $\mbf{78.22 \pm 0.47}$  & $\mbf{66.42 \pm 0.37}$ & $\mbf{82.62 \pm 0.31}$\\
\bottomrule
\end{tabular}
\end{sc}
\end{footnotesize}
% \end{small}
\end{center}
%\vskip -0.1in
\label{tab:sp_comp}
\end{table}

% \vskip -0.1cm
\cref{tab:sp_comp} compares the average test accuracy of \textsc{MAML}, \textsc{\leo{}} and \lsmetal{} with their conditional counterparts (denoted \textsc{SP + Method}) under the proposed structured prediction perspective. For consistency with our earlier discussion, we use \spml{} to denote \textsc{SP + \lsmetal{}}, although the framework is generally applicable to most methods. We observe that the structured prediction variants consistently outperform the original algorithms in all experiment settings. In particular, our approach improves \textsc{MAML} by the largest amount, averaging $\sim$6\% increase in test accuracy (e.g. from $48.70\%$ to $52.81\%$ for $5$-way-$1$-shot on \mimg{}). \textsc{\spml{}} averages  $\sim 1.5\%$ improvements over \lsmetal{}. We highlight that the structured prediction improves also \textsc{\leo{}} -- which is already a conditional meta-learning method -- by $\sim$1\%. This suggests that our structured prediction perspective is parallel to model-based conditional meta-learning, and might be combined to improve performance.

\subsection{Model Efficiency}
\label{sec:eff}
A potential limitation of \spml{} is the additional computational cost imposed by repeatedly performing the task-specific adaptation \cref{eq:estimator-improved}. Here we assess the trade-off between computations and accuracy improvement induced by this process. \cref{fig:sp_step} reports the classification accuracy of \spml{} when minimizing the structured prediction functional in \cref{eq:estimator-improved}, with respect to the number of training steps $J$, starting from the initialization points ($J=0$) learned via unconditional meta-learning using \lsmetal{} (refer to \cref{tab:sp_comp}). Therefore \cref{fig:sp_step} explicitly captures the additional performance improvements resulting from structured prediction. Aside from the slight decrease in performance on $1$-shot \mimg{} after $50$ steps, \spml{} shows a consistent and stable performance improvements over the $500$ steps via structured prediction. The results present a trade-off between performance improvements over unconditional meta-learning, and additional computational costs from optimizing structured prediction functional in \cref{eq:estimator-improved}.

More concretely, we quantify the actual time spent on structured prediction steps for \lsmetal{} in \cref{tab:cost}, which reports the average number of meta-gradient steps per second on a single Nvidia GTX $2080$. We note that $100$ steps of structured prediction for \lsmetal{} -- after which we observe the largest improvement in general -- take about 6 seconds to complete. In addition, \spml{} takes $\sim 0.23s$ for computing $\alpha(D)$ given a meta-train set of 30k tasks. In applications where model accuracy has the priority, the additional computational cost is a reasonable trade-off for accuracy. The adaptation cost is also amortized overall future queries in the adapted model. Lastly, we note that other conditional meta-learning methods also induce additional computational cost over unconditional formulation. For instance, \cref{tab:cost} shows that LEO performs fewer meta-gradient steps per second during training, as it incurs more computational cost for learning conditional initialization.

\subsection{Comparing Structured Prediction with Multi-Task Learning}\label{sec:exp_mtl}

We note that \spml{}'s formulation \cref{eq:estimator} is related to objective functions for multi-task learning (MTL) that also learn to weight tasks (e.g. \cite{chen2018gradnorm, kendall2018multi}). However, these strategies have entirely different design goals: MTL aims to improve performance on \textit{all} tasks, while meta-learning focuses only on the target task. This makes MTL unsuitable for meta-learning. To better visualize this fact, here we investigated whether MTL may be used as an alternative method to obtain task weighting for conditional meta-learning. We tested \cite{kendall2018multi} on \mimg{} and observed that the method significantly underperforms \spml{}, achieving $56.8 \pm 1.4$ for $1$-shot and $68.7 \pm 1.2$ for $5$-shot setting. We also observe that the performance on target tasks fluctuated widely during training with the MTL objective function, since MTL does not prioritize the performance of the target task, nor prevent negative transfer towards it.

\begin{figure}[t]
\begin{minipage}{0.49\textwidth}
\centering
\includegraphics[trim={0 0 0 1.2cm},clip, width=\textwidth]{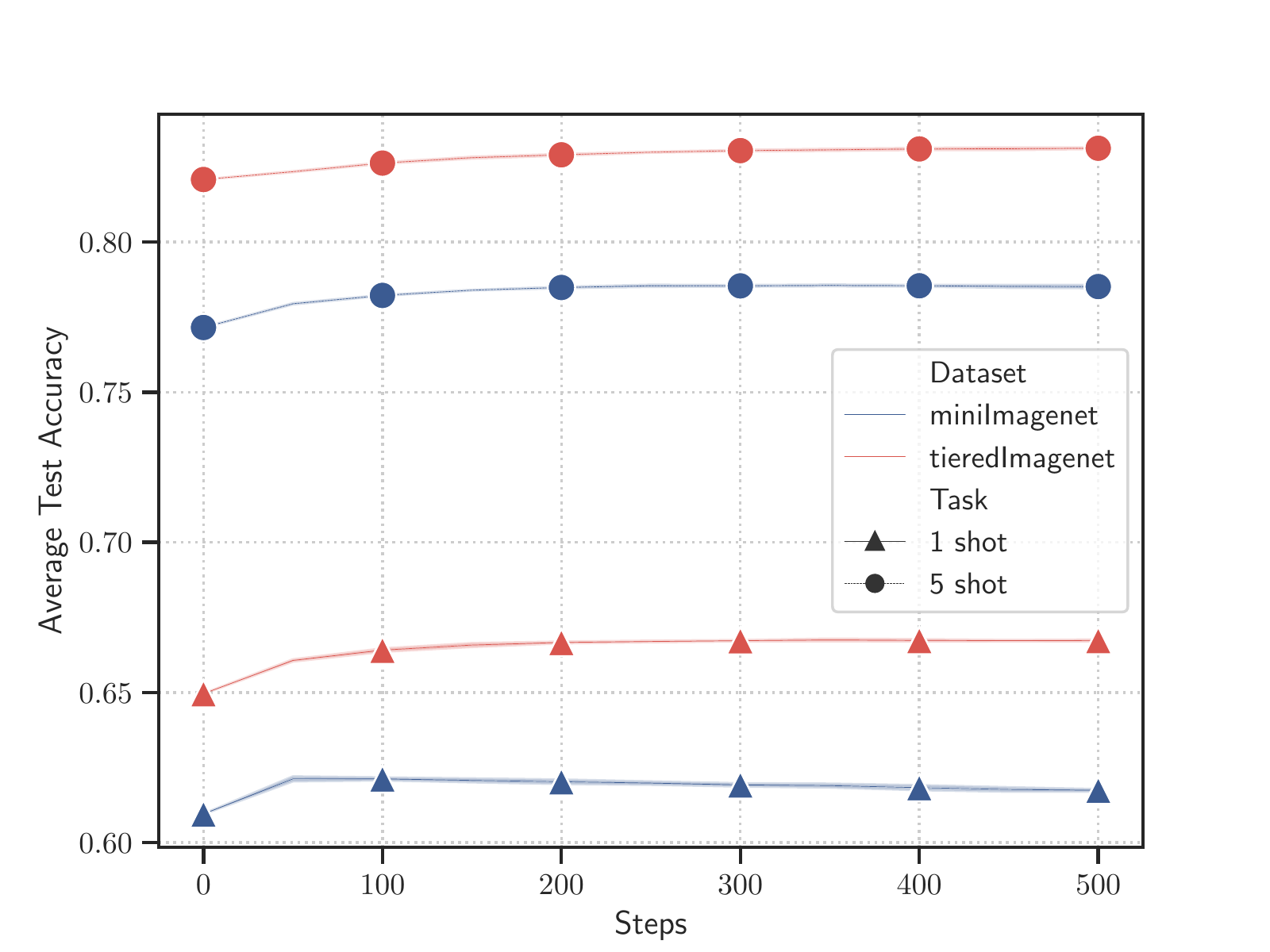}
\caption{Average test performance over 500 \spml{} structured prediction steps. }\label{fig:sp_step}
\end{minipage}
\hfill
\begin{minipage}[c]{0.49\textwidth}
\captionsetup{type=table} %% tell latex to change to table
    \begin{small}
        \begin{sc}
        \setlength{\tabcolsep}{4pt}
        \begin{tabular}{lcc}
        \toprule
            (steps/sec) & \mimg{}  & \timg{}\\
            \midrule
            \leo{} & $7.52 \pm 0.19$ & $6.95 \pm 0.47$\\
            \spml{} & $\mbf{17.82 \pm 0.27}$ & $\mbf{14.71 \pm 0.34}$\\
            \bottomrule
            \end{tabular}
        \end{sc}
    \end{small}
\caption{Meta-gradient steps per second on 5-shot learning tasks}\label{tab:cost}
\end{minipage}
\end{figure}

\section{Conclusion and Future Works}
\label{sec:con}
We proposed a novel perspective on conditional meta-learning based on structured prediction. Within this context, we presented task-adaptive structured meta-learning (\spml{}), a general framework that connects intimately to existing meta-learning methods via task-specific objective functions. The proposed method is theoretically principled, and empirical evaluations demonstrated its efficacy and effectiveness compared to the state of the art. In future work, we aim to design parametric approaches for \spml{} to improve model efficiency. We also aim to investigate novel metrics that better capture the similarity between datasets, given the key role played by the kernel function in our framework.

\section*{Broader Impact Statement}
Meta-learning aims to construct learning models capable of learning from experiences, Its intended users are thus primarily non-experts who require automated machine learning services, which may occur in a wide range of potential applications such as recommender systems and autoML. The authors do not expect the work to address or introduce any societal or ethical issues.

\section*{Acknowledgement}

The authors would like to thank the anonymous reviewers for their comments. This work was supported in part by UK DSTL/EPSRC Grant EP/P008461/1, National Science Scholarship from A{*}STAR, Singapore, the Royal Academy of Engineering Chair in Emerging Technologies to Y.D. and the Royal society of Engineering for grant SPREM RGS/R1/201149.

% {\itshape Authors are required to include a statement of the broader impact of their work, including its ethical
% aspects and future societal consequences. The authors should discuss both positive and negative outcomes,
% if any. For instance, authors should discuss a) who may benefit from this research, b) who may be
% put at disadvantage from this research, c) what are the consequences of failure of the system, and d)
% whether the task/method leverages biases in the data. If authors believe this is not applicable to them,
% authors can simply state this.
% Use unnumbered first level headings for this section, which should go at the end of the paper.}

{
\bibliographystyle{plain}
\bibliography{biblio}
}

\newpage 

\appendix

% Not Abbreviated
\crefname{assumption}{Assumption}{Assumptions}
\crefname{equation}{}{}
\Crefname{equation}{Eq.}{Eqs.}
\crefname{figure}{Figure}{Figures}
\crefname{table}{Table}{Tables}
\crefname{section}{Section}{Sections}
\crefname{theorem}{Theorem}{Theorems}
\crefname{proposition}{Proposition}{Propositions}
\crefname{fact}{Fact}{Facts}
\crefname{lemma}{Lemma}{Lemmas}
\crefname{corollary}{Corollary}{Corollaries}
\crefname{example}{Example}{Examples}
\crefname{remark}{Remark}{Remarks}
\crefname{algorithm}{Algorithm}{Algorithms}
\crefname{enumi}{}{}

\crefname{appendix}{Appendix}{Appendices}

\numberwithin{equation}{section}
\numberwithin{lemma}{section}
\numberwithin{proposition}{section}
\numberwithin{theorem}{section}
\numberwithin{corollary}{section}
\numberwithin{definition}{section}
\numberwithin{algorithm}{section}
% \numberwithin{fact}{section}
\numberwithin{remark}{section}

\onecolumn

\section*{\LARGE Supplementary Material: \spmlTitle{}}

The Appendix is organized in two main parts:
\begin{itemize}
    \item \cref{app:theory} proves the formal version of \cref{thm:rates-informal} and provides additional details on the connection between structured prediction and conditional meta-learning investigated in this work.
    \item \cref{app:details-and-experiments} provides additional details on the model hyperparameters and additional experimental evaluation.
    \item \cref{sec:additional-ablation} provides additional ablation studies.
\end{itemize}

\section{Structured Prediction for Conditional Meta-learning}
\label{app:theory}

We first recall the general formulation of the structured prediction approach in \cite{ciliberto2019localized}, followed by showing how the conditional meta-learning problem introduced in \cref{sec:conditional-meta-learning} can be cast within this setting. 

\subsection{General Structured Prediction}
\label{app:general-structured-prediction}

In this section, we borrow from the notation of \cite{ciliberto2019localized}. Consider $\X,\Y$ and $\Z$ three spaces, respectively the {\itshape input}, {\itshape label} and {\itshape output} sets of our problem. We make a distinction between label and output space since conditional meta-learning can be formulated within the setting described below by taking $\Z$ to be the meta-parameter space $\Theta$ and $\Y$ the space $\D$ of datasets. 

Structured prediction methods address supervised learning problems where the goal is to learn a function $f:\X\to\Z$ taking values in a ``structured'' space $\Z$. Here, the term structured is general and essentially encompasses output sets of strings, graphs, points on a manifold, probability distributions, etc. Formally, these are all spaces that are not linear or do not have a canonical embedding into a linear space $\R^k$. 

As we will discuss in the following, the lack of linearity on $\Z$ poses concrete challenges on modeling and optimization. In contrast, formally, the target learning problem is cast as a standard supervised learning problem of the form \cref{eq:supervised-risk}. More precisely, given a distribution $\rho$ on $\X\times\Y$
\eqal{\label{eq:struct-pred-risk}
    \min_{f:\X\to\Z} ~\E(f) \qquad \textrm{with} \qquad \E(f) = \int ~\sploss(f(x),y|x)~d\rho(x,y),
}
where $\sploss:\Z\times\Y\times\X\to\R$ is a loss function measuring prediction errors. Note that $\sploss(z,y|x)$ does not only compare the predicted output $z\in\Z$ with the label $y\in\Y$, but does that also depending or {\itshape conditioned} on the input $x\in\X$ (hence the notation $\sploss(z,y|x)$ rather than $\sploss(z,y,x)$). These conditioned loss functions were originally introduced to account for structured prediction settings where prediction errors depend also on properties of the input. For instance in ranking problems or in sequence-to-sequence translation settings, as observed in \cite{ciliberto2019localized}. 

\paragraph{Structured Prediction Algorithm\footnote{We note that in the original work, the authors considered a further parametrization of the loss $\sploss$ leveraging the concept of locality and parts. This led to the derivation of a more general (and involved) characterization of the estimator $\hat f$. However, for the setting considered in this work we consider a simplified scenario (see \cref{app:connection-sp-metalearning} below) and we can therefore restrict to the case where the loss does not assume a factorization into parts, namely the set of parts $P$ corresponds to $P = \{1\}$ the singleton, leading to the structured prediction estimator \cref{eq:original-sp-estimator}.}}
Given a finite number $n\in\N$ of points $(x_i,y_i)_{i=1}^n$ independently sampled from $\rho$, the structured prediction algorithm proposed in \cite{ciliberto2019localized} is an estimator $\hat f:\X\to\Z$ such that, for every $x\in\X$
\eqal{\label{eq:original-sp-estimator}
    \hat f(x) ~=~ \argmin_{z\in\Z}~ \sum_{i=1}^n \alpha_i(x)~\sploss(z,y_i|x_i).    
}
where, given a reproducing kernel $k:\X\times\X\to\R$, the weighs $\alpha$ are obtained as 
\eqal{\label{eq:origina-sp-alphas}
    \alpha(x) = (\alpha_1(x),\dots,\alpha_n(x))^\top \in \R^n \qquad \textrm{with} \qquad \alpha(x) = (\mbf{K}+\lambda I)^{-1}~v(x),
}
where $\mbf{K}\in\R^{n \times n}$ is the empirical kernel matrix with entries $K_{ij} = k(x_i,x_j)$ and $v(x)\in\R^n$ is the evaluation vector with entries $v(x)_i = k(x,x_i)$, for any $i,j=1,\dots,n$ and $\lambda>0$ is a hyperparameter. 

The estimator above has a similar form to the \spml{} algorithm proposed in this work in \cref{eq:estimator}. In the following, we show that the latter is indeed a special case of \cref{eq:original-sp-estimator}.

\subsection{A Strucutred Prediction perspective on Conditional Meta-learning}
\label{app:connection-sp-metalearning}

In the conditional meta-learning setting introduced in \cref{sec:conditional-meta-learning} the goal is to learn a function $\tau:\D\to\Theta$ where $\D$ is a space of datasets and $\Theta$ a space of learning algorithms. We define the conditional meta-learning problem according to the expected risk \cref{eq:conditional-meta-learning-risk} as
\eqal{\label{eq:conditional-meta-learning-problem}
    \min_{\tau:\D\to\Theta}~\E(\tau) \qquad \textrm{with} \qquad \E(\condmetapar) = \int ~\Lagr\Big(\alg\big(\tau(D^{tr}),D^{tr}\big),~D^{val}~\Big)~d\pi(D^{tr},D^{val}),
}
where $\pi$ is a probability distribution sampling the pair of train and validation datasets $D^{tr}$ and $D^{val}$. We recall that the distribution $\pi$ samples the two datasets according to the process described in \cref{sec:bg}, namely by first sampling $\rho$ a task-distribution (on $\X\times\Y$) from $\mu$ and then obtaining $D^{tr}$ and $D^{val}$ by independently sampling points $(x,y)$ from $\rho$. Therfore $\pi = \pi_\mu$ can be seen as implicitly induced by $\mu$. In practice, we have only access to a meta-training set $\metaD = (D_i^{tr},D_i^{val})_{i=1}^N$ of train-validation pairs sampled from $\pi$. 

We are ready to formulate the conditional meta-learning problem within the structured prediction setting introduced in \cref{app:general-structured-prediction}. In particular, we take the input and label spaces to correspond to the set $\D$ and choose as output set the space $\Theta$ of meta-parameters. In this setting, the loss function is of the form $\sploss:
\Theta\times\D\times\D\to\R$ and corresponds to 
\eqal{\label{eq:loss-equivalence}
    \sploss(\theta,D^{val} | D^{tr}) ~=~ \Lagr\Big(\alg\big(\theta,D^{tr}\big),~D^{val}~\Big).
}
Therefore, we can interpret the loss $\sploss$ as the function measuring the performance of a meta-parameter $\theta$ when the corresponding algorithm $\alg(\theta,\cdot)$ is trained on $D^{tr}$ and then tested on $D^{val}$. Under this notation, it follows that \cref{eq:conditional-meta-learning-problem} is a special case of the structured predition problem \cref{eq:struct-pred-risk}. Therefore, casting the general structured prediction estimator \cref{eq:original-sp-estimator} within this setting yields the \spml{} estimator proposed in this work and introduced in \cref{eq:estimator}, namely  $\condmetapar_N:\D\to\Theta$ such that, for any dataset $D\in\D$,
\eqals{
    \condmetapar_N(D) ~=~ \argmin_{\theta\in\Theta}~ \sum_{i=1}^N~ \alpha_i(D)~\Lagr\Big(\alg\big(\theta,D^{tr}\big),~D^{val}~\Big),
}
where $\alpha:\D\to\R^N$ is learned according to \cref{eq:origina-sp-alphas}, namely 
\eqals{
 \alpha(x) = (\alpha_1(x),\dots,\alpha_N(x))^\top \in \R^N \qquad \textrm{with} \qquad \alpha(x) = (\mbf{K}+\lambda I)^{-1}~v(D),
}
with $\mbf{K}$ and $v(D)$ defined as in \cref{eq:estimator}. Hence, we have recovered $\condmetapar_N$ as it was introduced in this work. 

\subsection{Theoretical Analysis}

In this section we prove \cref{thm:rates}. Our result can be seen as a corollary of \cite[Thm.5]{ciliberto2020general} applied to the generalized structured prediction setting of \cref{app:general-structured-prediction}. The result hinges on two regularity assumptions on the loss $\sploss$ and on the meta-distribution $\pi$ that we introduce below.

\begin{assumption}\label{asm:loss-regularity}
The loss $\sploss$ is of the form \cref{eq:loss-equivalence} and admits derivatives of any order, namely $\sploss\in C^{\infty}(\Z\times\Y\times\X)$. 
\end{assumption}

Recall that by \cref{eq:loss-equivalence} we have 
\eqal{
    \Lagr(\theta,D^{val},D^{tr}) = \frac{1}{|D^{val}|}~\sum_{(x,y)\in D^{val}}~ \ell\Big(~\big[\alg(\theta,D^{tr})\big](x),~y\Big).
}
Therefore, sufficient conditions for \cref{asm:loss-regularity} to hold are: $i)$ the inner loss function $\ell$ is smooth (e.g. least-squares, as in this work) and $ii)$ the inner algorithm $\alg(\cdot,\cdot)$ is smooth both with respect to the meta-parameters $\theta$ and the training dataset $D^{tr}$. For instance, in this work, \cref{asm:loss-regularity} is verified if the meta-representation network $\psi_\theta$ is smooth with respect to the meta-parametrization $\theta$. Indeed, $\ell$ is chosen to be the least-squares loss and the closed form solution $W(\theta,D^{tr}) = X_\theta^\top(X_\theta X_\theta^\top + \lambda I)^{-1}Y$ in \cref{eq:ls-closed-form} is smooth for any $\lambda>0$.

The second assumption below concerns the regularity properties of the meta-distribution $\pi$ and its interaction with the loss $\sploss$. The assumption leverages the notion of Sobolev spaces. We recall that for a set $\mathcal{K}\subset\R^{d}$ the Sobolev space $W^{s,2}(\mathcal{K})$ is the Hilbert space of functions from $\mathcal K$ to $\R$ that have square integrable weak derivatives up to the order $s$. We recall that if $\mathcal{K}$ satisfies the cone condition, namely there exists a finite cone $C$ such that each $x\in\mathcal{K}$ is the vertex of a cone $C_x$ contained in $\mathcal{K}$ and congruent to $C$ \cite[Def. $4.6$]{adams2003sobolev}, then for any $s>d/2$ the space $W^{s,2}(\mathcal{K})$ is a RKHS. This follows from the Sobolev embedding theorem \cite[Thm. 4.12]{adams2003sobolev} and the properties of RKHS (see e.g. \cite{berlinet2011reproducing} for a detailed proof).

Given two Hilbert spaces $\hh$ and $\F$, we denote by $\hh\otimes\F$ the tensor product of $\hh$ and $\F$. In particular, given two basis $(h)_{i\in\N}$ and $(f_j)_{j\in\N}$ for $\hh$ and $\F$ respectively, we have 
\eqals{
    \scal{ h_i\otimes f_j}{h_{i'}\otimes f_{j'}}_{\hh\otimes\F} = \scal{h_i}{h_{i'}}_\hh \cdot \scal{f_j}{f_{j'}}_\F,
}
for every $i,i',j,j'\in\N$. We recall that $\hh\otimes\F$ is a Hilbert space and it is isometric to the space $\hs(\F,\hh)$ of Hilbert-Schmidt (linear) operators from $\F$ to $\hh$ equipped with the standard Hilbert-Schmidt $\scal{\cdot}{\cdot}_\hs$ dot product. In the following, we denote by $\msf{T}:\hh\otimes\F \to \hs(\F,\hh)$ the isometry between the two spaces. 

We are ready to state our second assumption.

\begin{assumption}\label{asm:regularity-conditional-mean-embedding}
Assume $\Theta\subset\R^{d_1}$ and $\D\subset\R^{d_2}$ compact sets satisfying the cone condition and assume that there exists a reproducing kernel $k:\D\times\D\to\R$ with associated RKHS $\F$ and $s>(d_1 + 2d_2)/2$ such that the function $\gstar:\D\to\hh$ with $\hh = W^{s,2}(\Theta\times\D)$, characterized by 
\eqal{\label{eq:regular-gstar}
    \gstar(D^{tr}) = \int \sploss(\cdot,D^{val} |~ \cdot)~d\pi(D^{val}|D^{tr}) \qquad\qquad \forall D^{tr}\in\D,
}
is such that $\gstar\in \hh\otimes\F$ and, for any $D\in\D$, we have that the application of the operator $\msf{T}(\gstar):\F\to\hh$ to the function $k(D,\cdot)\in\F$ is such that $\msf{T}(\gstar) ~ k(D,\cdot) = \gstar(D)$. 
\end{assumption}

The function $\gstar$ in \cref{eq:regular-gstar} can be interpreted as capturing the interaction between $\sploss$ and the meta-distribution $\pi$. In particular, \cref{asm:regularity-conditional-mean-embedding} imposes two main requirements: $i)$ for any $D\in\D$ the output of $\gstar$ is a vector in a Sobolev space (i.e. a function) of smoothness $s> (d_1 + 2d_2)/2$, namely $\gstar(D) \in W^{s,2}(\Theta\times\D)$ and, $ii)$ we require $\gstar$ to correspond to a vector in $W^{s,2}(\Theta\times\D)\otimes\F$. Note that the first requirement is always satisfied if \cref{asm:loss-regularity} holds. The second assumption is standard in statistical learning theory (see \cite{shalev2014understanding,caponnetto2007} and references therein) and can be interpreted as requiring the conditional probability $\pi(\cdot|D^{tr})$ to not vary dramatically for small perturbations of $D^{tr}$. 

We are ready to state and prove our main theorem, whose informal version is reported in \cref{thm:rates-informal} in the main text.

\begin{restatable}[Learning Rates]{theorem}{TRates}\label{thm:rates}
Under \cref{asm:loss-regularity,asm:regularity-conditional-mean-embedding}, let $\metaD = (D^{tr}_i,D^{val}_i)_{i=1}^N$ be a meta-training set of points independently sampled from a meta-distribution $\pi$. Let  $\condmetapar_N$ be the estimator in \cref{eq:estimator} trained with $\lambda_2 = N^{-1/2}$ on $\metaD$. Then, for any $\delta\in(0,1]$ the following holds with probability larger or equal than $1-\delta$,
\eqal{
    \E(\condmetapar_N) ~- \inf_{\condmetapar:\D\to\Theta}~\E(\condmetapar)~\leq ~c \log(1/\delta)~ N^{-1/4},
}
where $c$ is a constant depending on $\kappa^2 = \sup_{D\in\D}k(D,D)$ and $\nor{\gstar}_{\hh\otimes\F}$ but independent of $N$ and $\delta$.
\end{restatable}
    
\begin{proof}
Let $\hh = W^{s,2}(\Theta\times\D)$ and $\G = W^{s,2}(\D)$. Since $s>(d_1+2d_2)/2$, both $\G$ and $\hh$ are reproducing kernel Hilbert spaces (RKHS) (see discussion above or \cite{berlinet2011reproducing}). Let $\psi:
\Theta\times\D\to\hh$ and $\varphi:\D\to\G$ be two feature maps associated to $\hh$ and $\G$ respectively. Without loss of generality, we can assume the two maps to be normalized.

We are in the hypotheses\footnote{the original theorem was applied to the case where $\Z\times\X=\Y$ was the probability simplex in finite dimension. However the proof of such result requires only that $\hh$ and $\G$ are RKHS and can therefore be applied to the general case where $\Z\times\X$ and $\Y$ are different from each other and they do not correspond to the probability simplex but are rather subset of $\R^k$ (possibly with different dimension for each space) and satisfy the boundary condition \cite{berlinet2011reproducing}. Therefore in our setting we can take $\Z = \Theta$ and $\X = \Y = \D$ to obtain the desired result.} of \cite[Thm. $6$ Appendix D]{luise2018differential}, which guarantees the existence of a Hilbert-Schmidt operator $V:\G\to\hh$, such that $\sploss$ can be characterized as 
\eqal{\label{eq:metal-loss-is-self}
    \sploss(\theta,D^{val} | ~ D^{tr}) ~=~ \scal{\psi(\Theta,D^{tr})}{V\varphi(D^{val})}_\hh
}
for any $D^{tr},D^{val}\in\D$ and $\theta\in\Theta$. Since the feature maps $\varphi$ and $\psi$ are normalized \cite{berlinet2011reproducing}, this implies also $\nor{V}_{\hs} = \nor{\sploss}_{s,2} < +\infty$, namely that the Sobolev norm of $\sploss$ in $W^{s,2}(\Theta\times\D\times\D)$ is equal to the Hilbert-Schmidt norm of $V$.

The result in \cref{eq:metal-loss-is-self} corresponds to the definition of {\itshape Structure Encoding Loss Function (SELF)} in \cite[Def. $1$]{ciliberto2019localized}. Additionally, if we denote $\widetilde\varphi = V\varphi$, we obtain the equality
\eqal{\label{eq:equivalent-assumption}
\gstar(D^{tr}) ~=~ 
    \int \widetilde\varphi(D^{val})~d\pi(D^{val}|D^{tr}) ~=~ \int \sploss(\cdot,D^{val}|\cdot)~d\pi(D^{val}|D^{tr}),
}
for all $D^{tr}\in\D$, where $\gstar:\D\to\hh$ is defined as in \cref{asm:regularity-conditional-mean-embedding}, we are in the hypotheses of the comparison inequality theorem \cite[Thm. $9$]{ciliberto2019localized}. In our setting, this result states that for any measurable function $g:\D\to\hh$ and the corresponding function $\condmetapar_g:\D\to\Theta$ defined as 
\eqal{
	\condmetapar_g(D) = \argmin_{\theta\in\Theta}~ \scal{\psi(\theta,D)}{g(D)}_\hh \qquad \forall D\in\D, 
}
we have 
\eqal{\label{eq:comparison-inequality}
	\E(\condmetapar_g) - \inf_{\condmetapar:\D\to\Theta}\E(
	\condmetapar)~\leq~ \sqrt{\int \nor{g(D) - \gstar(D)}_\hh^2 ~ d\pi_{\D}(D)},
}
where $\pi_\D(D^{tr})$ denotes the marginal of $\pi(D^{val},D^{tr})$ with respect to training data. Note that the constant $\msf{c}_\sploss$ that appears in the original comparison inequality is upper bounded by $1$ in our setting since $\msf{c}_\sploss = \sup_{D,\theta}\nor{\psi(\theta,D)}$ and the feature map $\psi$ is normalized. 

Let now $g_N:\D\to\Theta$ be the minimizer of the vector-valued least-squares empirical risk minimization problem
\eqals{
	g_N = \argmin_{g\in\hh\otimes\F}~ \frac{1}{N}~ \sum_{i=1}^N~ \nor{g(D_i^{tr}) - \tilde\varphi(D_i^{val})}_\hh^2 + \lambda_2 \nor{g}_{\hh\otimes\F}^2. 	
}
This problem can be solved in closed form and it can be shown \cite[Lemma B.$4$]{ciliberto2020general} that $g_N$ is of the form 
\eqal{
	g_N(D) = \sum_{i=1}^n \alpha_i(D)~ \tilde\varphi(D_i^{val}),
}
for all $D\in\D$, where $\alpha_i(D)$ is defined as in \cref{eq:estimator}. Due to linearity (see also Lemma $8$ in \cite{ciliberto2019localized}) we have
\eqal{
	\condmetapar_{g_N}(D) & = \argmin_{\theta\in\Theta}~ \scal{\psi(\theta,D)}{g_N(D)}_\hh \\
	& = \argmin_{\theta\in\Theta}~ \sum_{i=1}^N~ \alpha_i(D)~\Lagr\Big(\alg\big(\theta,D^{tr}\big),~D^{val}~\Big) \\
	& = \condmetapar_N(D),
}
which corresponds to the estimator $\condmetapar_N(D)$ studied in this work and introduced in \cref{eq:estimator}. The comparison inequality \cref{eq:comparison-inequality} above, becomes 
\eqal{
\E(\condmetapar_N) - \inf_{\condmetapar:\D\to\Theta}\E(f)~\leq~ \sqrt{\int \nor{g_N(D) - \gstar(D)}_\hh^2 ~ d\pi_{\D}(D)}.
}
Therefore, we can obtain a learning rate for the excess risk of $\condmetapar_N$ by studying how well the vector-valued least-squares estimator $g_N$ is approximating $\gstar$. Since $\gstar\in\hh\otimes\F$ from the hypothesis, we can replicate the proof in \cite[Thm. $5$]{ciliberto2020general} to obtain the desired result. Note that by framing our problem in such context we obtain a constant $c$ that depends only on the norm of $\gstar$ as a vector in $\hh\otimes\F$. We recall that $\gstar$ captures the ``regularity'' of the meta-learning problem. Therefore, the more regular (i.e. easier) the learning problem, the faster the learning rate of the proposed estimator. 
\end{proof}

\section{Model and Experiment Details}
\label{app:details-and-experiments}
We provide additional details on the model architecture, experiment setups, and hyperparameter choices. We performed only limited mode tuning, as it is not the focus on the work.

\subsection{Details on \lsmetal{}}
\label{app:ls-meta}
Meta-representation learning methods formulate meta-learning as the process of finding a shared representation to be fine-tuned for each task. Formally, they model the task predictor as a composite function $f_W \circ \psi_\theta:\X\to\Y$, with $\psi_\theta:\X\to\R^p$ a shared feature extractor meta-parametrized by $\theta$, and $f_W:\R^p\to\Y$ a map parametrized by $W$. The parameters $W$ are learned for each task as a function $W(\theta,D)$ via the inner algorithm
\eqal{\label{eq:meta-representation-model}
f_{W(\theta,D)}\circ\psi_\theta ~=~ \alg(\theta,D).
}

\cite{bertinetto2018meta} proposed $\alg(\theta,D)$ to perform empirical risk minimization of $f_W$ over $D = (x_i,y_i)_{i=1}^m$ with respect to the least-squares loss $
\ell(y,y') = \nor{y - y'}^2$. Assuming\footnote{For instance, $C$ is the total number of classes, and $y\in\Y$ the one-hot encoding of a class in classification tasks} $\Y=\R^C$ and a linear model for $f_W$, this corresponds to performing ridge-regression on the features $\psi_\theta$, yielding the closed-form solution
\eqal{\label{eq:ls-closed-form}
    W(\theta,D) ~=~ X_\theta^\top (X_\theta X_\theta^\top + \lambda_{\theta} I)^{-1} Y,
}
where $\lambda_1>0$ is a regularizer. $X_\theta\in\R^{m\times p}$ and $Y\in\R^{m\times C}$ are matrices with $i$-th row corresponding to the $i$-th training input $\psi_\theta(x_i)$ and output $y_i$ in the dataset $D$, respectively. The closed-form solution \cref{eq:ls-closed-form} has the advantage of being $i)$ efficient to compute and $ii)$ suited for the computation of meta-gradients with respect to $\theta$. Indeed, $\nabla_\theta W(\theta,D)$ can be computed in closed-form or via automatic differentiation.

\lsmetal{} is a meta-representation learning algorithm consists of:
\begin{itemize}[topsep=0em,itemsep=0em,partopsep=0em]
    \item The {\itshape meta-representation} architecture $\psi_\theta:\X\to\R^p$ is a two-layer fully-connected network with residual connection~\cite{he2016deep}.

    \item The {\itshape task predictor} $f_W:\R^{p}\to\Y$ is a linear model  $f_W\big(\psi_\theta(x)\big) = W\psi_\theta(x)$ with  $W\in\R^{C\times p}$ the model parameters. We assume $\Y=\R^C$ (e.g. one-hot encoding of $C$ classes in classification settings).
    
    \item The {\itshape inner algorithm} is $f_{W(\theta,D)}\circ\psi_\theta = \alg(\theta,D)$, where $W(\theta,D)$ is the least-squares closed-form solution introduced in \cref{eq:ls-closed-form}.
    
\end{itemize}
We note that \cite{bertinetto2018meta} uses the cross-entropy $\ell$ to induce $\Lagr$. Consequently, when optimizing the meta-parameters $\theta$, the performance of $W(\theta, D)$ is measured on a validation set $D'$ with respect to a loss function (cross-entropy) different from the one used to learn it (least-squares). We observe that such incoherence between inner- and meta-problems lead to worse performance than least-square task loss.

\subsection{Model Architecture}
Given the pre-trained representation $\varphi(x)\in\R^{640}$, the proposed model is $f_{\theta}(\varphi(x))= \varphi(x) + g_{\theta}(\varphi(x))$, a residual network with fully-connected layers. Each layer of the fully-connected network $g_{\theta}(\varphi(x))$ is also 640 in dimension.

We added a $\ell_2$ regularization term on $\theta$, with a weight of $\lambda_\theta$ reported below.

For top-$M$ values from $\alpha(D)$, we normalize the values such that they sum to 1.

\subsection{Experiment Setups}
We use the same experiment setup as \textsc{LEO}~\cite{rusu2018meta} by adapting its official implementation\footnote{\url{https://github.com/deepmind/leo}}. For both $5$-way-$1$-shot and $5$-way-$5$-shot settings, we use the default environment values from the implementation, including a meta-batch size of 12, and 15 examples per class for each class in $D^{val}$ to ensure a fair comparison.

\subsection{Model Hyperparameters}
Models across all settings share the same hyperparameters, listed in \cref{tab:hyper}.

\begin{table}[t]
%\vskip 0.15in
\caption{Hyperparameter values used in the experiments}
\begin{center}
\begin{small}
\begin{sc}
\begin{tabular}{llc}
\toprule
Symbol & Description & Values\\
\midrule
$\lambda$ in \cref{eq:estimator} & regularizer for learning $\alpha(D)$ & $10^{-8}$\\
$\lambda_{\theta}$ in \cref{eq:ls-closed-form} & regularizer for the least-square solver,  & $0.1$\\
$\sigma$ in \cref{eq:dataset-signature} & kernel bandwidth & $50$\\
$\eta$ & meta learning rate & $10^{-4}$\\
$N$ & total number of meta-training tasks & $30,000$\\
$M$ & number of tasks to keep in \cref{alg:spml} & $500$\\
\end{tabular}
\end{sc}
\end{small}
\end{center}
%\vskip -0.1in
\label{tab:hyper}
\end{table}

\section{Additional Ablation Study}\label{sec:additional-ablation}
\subsection{Structured Prediction from Random Initialization}
\label{app:no-init}
We note that the unconditional initialization of $\theta$ from \cref{sec:practical} is optional and designed for improving computational efficiency. \cref{fig:sp_no_init} reports how \spml{} performs, starting from random initialization. The results suggest that structured prediction takes longer to converge with random initialization, but achieves performance comparable to \cref{tab:comp}. In addition, structured prediction appears to work well, despite having access to only a small percentage of meta-training tasks ($M=1000$ in this experiment).

\begin{figure}[t]
\centering
\includegraphics[trim={0 0 0 1.2cm},clip,width=0.5\textwidth]{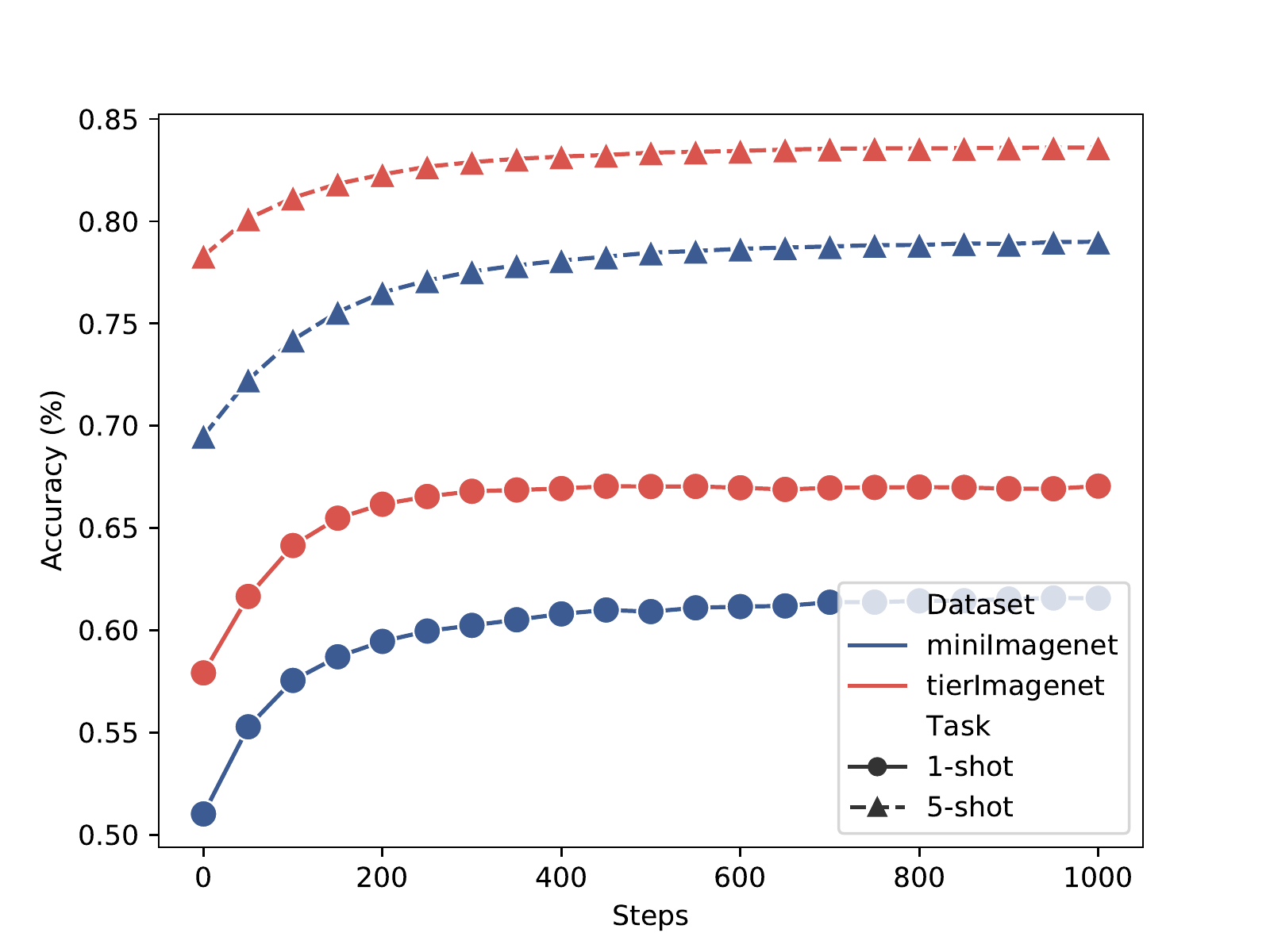}
\caption{Average test task performance over 1000 structured prediction steps. Structured prediction takes longer to converge from random initialization, but achieves comparable performance.}
\label{fig:sp_no_init}
\end{figure}

\subsection{Top-M Filtering}
\label{app:sp_top_m}
\cref{tab:sp_m_filter} reports the performance of structured prediction by varying the number $M$ of tasks used. We use $5$-way-$5$-shot on \mimg{} and \timg{} as the experiment settings.

\begin{table}[t]
    \centering
    \setlength{\tabcolsep}{4pt}
    \caption{The effects of top-$M$ filtering on structured prediction accuracy. \spml{} is robust to the choices of $M$.}
    \label{tab:sp_m_filter}
    \begin{tabular}{cccccc}
    $M$ & 100 & 500 & 1000 & 10000 & 30000\\
    \midrule
    \mimg{} (\%) & $77.60 \pm 0.30$ & $78.22 \pm 0.47$ & $78.43 \pm 0.39$ & $78.47 \pm 0.37$ & $78.51 \pm 0.42$\\
    \timg{} (\%) & $81.95 \pm 0.23$ & $82.62 \pm 0.31$ & $82.95 \pm 0.27$ & $83.01 \pm 0.29$ & $83.03 \pm 0.35$\\
    \end{tabular}
    % \vskip 0.3cm
\end{table}
% \vskip -0.3cm

The results show that \spml{} is robust to the choice of $M$. As $M$ increases, its impact on performance is small as most tasks have tiny weights with respect to the objective function.

\subsection{Explicit Dependence on Target Tasks}
\label{app:estimator_ablation}
In \cref{eq:estimator-improved}, we introduced an additional $\Lagr\big(\alg(\theta, D), D\big)$, such that the objective function explicitly depends on target task $D$. To study the contribution of each term towards test accuracy, we modify \cref{eq:estimator-improved} by weighting the contribution of each term by the weights $\beta_1,\beta_2\geq0$
\eqal{\label{eq:estimator-improved-lam}
    \condmetapar(D) = \argmin_{\theta \in \Theta} ~ \beta_1\sum_{i=1}^N \alpha_i(&D) ~ \Lagr\big(\alg(\theta, D^{tr}_i),~ D^{val}_i\big) + \beta_{2} \Lagr\big(\alg(\theta, D),~ D\big),
}
We report the test accuracy on \mimg{} and \timg{} on $5$-way-$5$-shot below.

\begin{table}[t]
    \centering
    \begin{tabular}{cccccc}
     & $\beta_{1}=0,\beta_2=1$ & $\beta_{1}=1,\beta_2=0$ & $\beta_{1}=1,\beta_2=1$ & $\beta_{1}=1,\beta_2=2$\\
    \midrule
    \mimg{} (\%) & $73.59 \pm 0.49$ & $77.32 \pm 0.36$ & $78.22 \pm 0.47$ & $78.51 \pm 0.32$\\
    \timg{} (\%) & $79.74 \pm 0.62$ & $81.63 \pm 0.47$ & $82.62 \pm 0.31$ & $83.01 \pm 0.43$
    \end{tabular}
    % \vskip 0.2cm
    \caption{Test accuracy on \mimg{} and \timg{} by adjusting the importance of each term in \cref{eq:estimator-improved-lam}}
    \label{tab:sp_special}
\end{table}
% \vskip -0.3cm

\cref{tab:sp_special} suggests that the explicit dependence on target task $D$, \textbf{combined with} other relevant tasks, provides the best training signal for \spml{}. In particular, optimizing with respect to the target task alone (i.e. $\beta_{1}=0,\beta_2=1$) leads to overfitting while excluding the target task (i.e. $\beta_{1}=1,\beta_2=0$) ignores valuable training signal, leading to underfitting. Ultimately, both extremes lead to sub-optimal performance. The results in \cref{tab:sp_special} show that both terms in \cref{eq:estimator-improved} are necessary to achieve good test accuracy.

\subsection{Choice of Kernel for Structured Prediction}
\label{app:sp_kernel}
Lastly, we study how the choice of kernel from \cref{eq:dataset-signature} affects test accuracy. In addition to the Gaussian kernel considered in this work, we include both the linear kernel
\eqals{
    k(D,D') = \scal{\Phi(D)}{\Phi(D')} + c
}
with $c>0$ a hyperparameter, and the Laplace kernel
\eqals{
    k(D,D') = \exp(-\nor{\Phi(D) - \Phi(D')}/\sigma)
}
with $\sigma>0$ a hyperparameter. We considered the $5$-way-$5$-shot task on \mimg{} and \timg{} to compare the impact of the kernel on \spml{}.

\begin{figure}[t]
     \subfloat[\mimg{}]{%
      \includegraphics[trim={0 0 0 1.2cm},clip,width=0.45\textwidth]{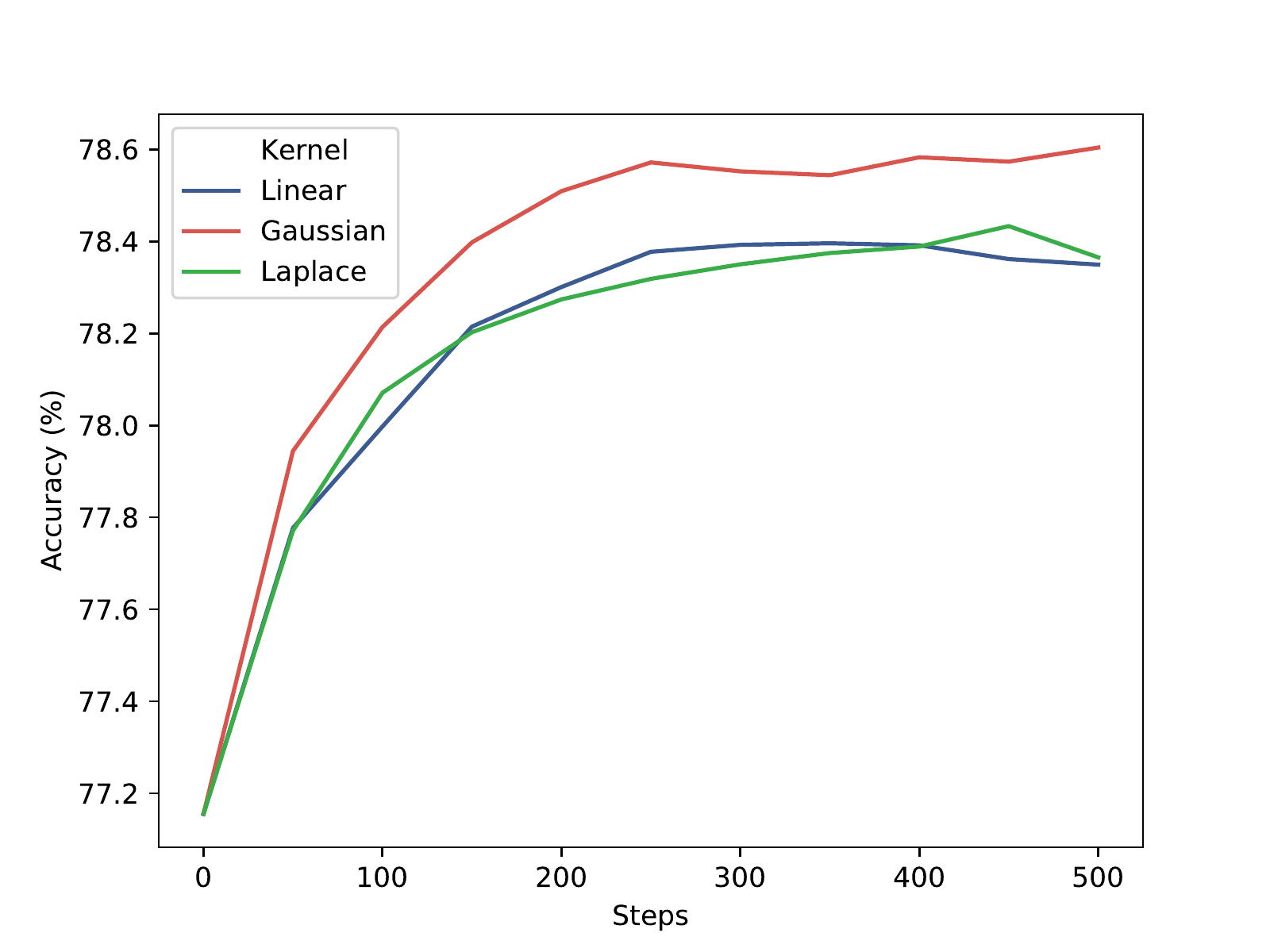}
     }
     \hfill
     \subfloat[\timg{}]{%
      \includegraphics[trim={0 0 0 1.2cm},clip,width=0.45\textwidth]{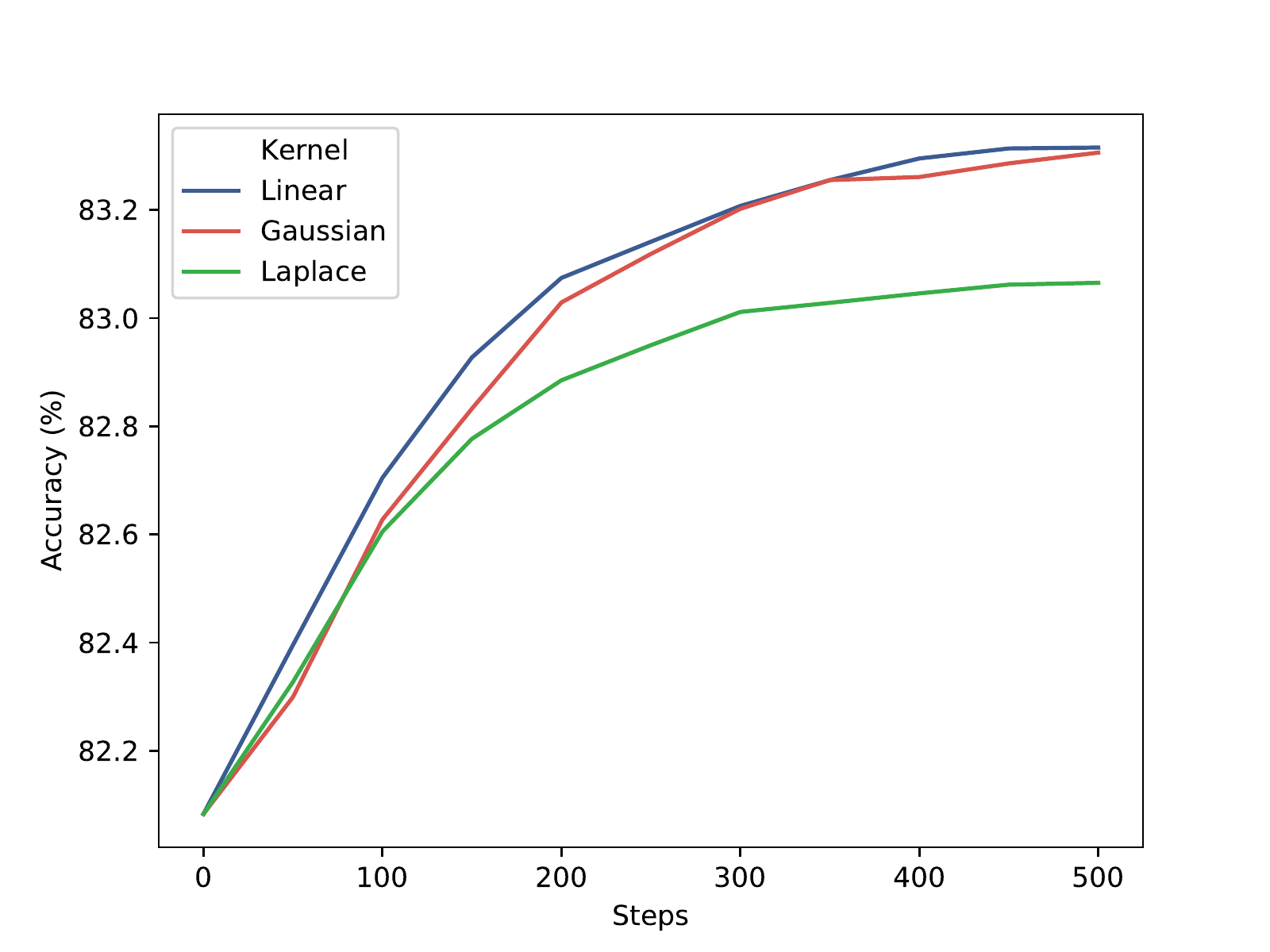}
     }
\caption{Test performance with different choices of kernel for structured prediction. Gaussian kernel obtains the best performance}
\label{fig:sp_comp_kernel}
\end{figure}

\cref{fig:sp_comp_kernel} shows that the Gaussian kernel overall obtains the best performance among the three candidates. In our experiments, we observed that Gaussian kernels are most robust with respect to bandwidth parameters, while Laplace kernels appeared sensitive to the bandwidth parameter. Careful model selection for the bandwidth parameter might lead to better or comparable performance, but it is beyond the scope of this work. In addition, we observed the linear kernel to perform well in some settings but less expressive in general.

\end{document}

%% file: macro-arxiv.tex
\declaretheorem[name=Remark]{remark}

\declaretheorem[name=Assumption,refname=Asm.]{assumption}

\definecolor{dgreen}{rgb}{0.00,0.49,0.00}
\definecolor{dblue}{rgb}{0,0.08,0.75}
\hypersetup{
    colorlinks = true,
    citecolor=dgreen,
    urlcolor=dblue,
    linkcolor=dblue
}

\crefname{assumption}{Assumption}{Assumptions}
\crefname{equation}{}{}
\crefname{figure}{Fig.}{Figs.}
\crefname{table}{Tab.}{Tabs.}
\crefname{section}{Sec.}{Sec.}
\crefname{theorem}{Thm.}{Thm.}
\crefname{lemma}{Lemma}{Lemmas}
\crefname{corollary}{Cor.}{Cor.}
\crefname{example}{Example}{Examples}
\crefname{remark}{Remark}{Remarks}
\crefname{algorithm}{Alg.}{Algorightms}

\crefname{appendix}{Appendix}{Appendices}
\crefname{subappendix}{Appendix}{Appendices}
\crefname{subsubappendix}{Appendix}{Appendices}

\newcommand{\spmlTitle}{Structured Prediction for Conditional Meta-Learning}
\newcommand{\spml}{TASML}
\newcommand{\spmlong}{Task-adaptive Structured Meta-Learning}

\newcommand{\lsmetal}{{\textsc{LS~Meta-Learn}}}
\newcommand{\imgnet}{\textsc{ImageNet}}
\newcommand{\mimg}{\textit{mini}\imgnet{}}
\newcommand{\timg}{\textit{tiered}\imgnet{}}
\newcommand{\cifar}{\textsc{CIFAR-FS}}
\newcommand{\leo}{{LEO}}
\newcommand{\sploss}{{\bigtriangleup}}
\newcommand{\msf}[1]{\mathsf{#1}}
\newcommand{\mbf}[1]{\mathbf{#1}}

\newcommand{\R}{{\mathbb{R}}}
\newcommand{\N}{{\mathbb{N}}}

\newcommand{\EE}{\mathbb{E}}
\newcommand{\Lagr}{\mathcal{L}}

\newcommand{\argmin}{\operatornamewithlimits{argmin}}

\newcommand{\X}{{\mathcal{X}}}
\newcommand{\Y}{{\mathcal{Y}}}
\newcommand{\Z}{{\mathcal{Z}}}
\newcommand{\F}{{\mathcal{F}}}
\newcommand{\D}{{\mathcal{D}}}

\newcommand{\E}{{\mathcal{E}}}

\newcommand{\hh}{{\mathcal{H}}}
\newcommand{\G}{{\mathcal{G}}}

\newcommand{\hs}{{\textrm{HS}}}

\newcommand{\metaD}{{S}}

\newcommand{\alg}{\textrm{Alg}}
\newcommand{\jointAlg}{T}

\newcommand{\condmetapar}{\tau}

\renewcommand{\paragraph}[1]{~\newline\noindent{\bfseries #1.}}

\newcommand{\gstar}{{g^*}}

\newcommand{\eqals}[1]{\begin{align*}#1\end{align*}}

\newcommand{\eqal}[1]{\begin{align}#1\end{align}}

\providecommand{\scal}[2]{\left\langle{#1},{#2}\right\rangle}

\providecommand{\nor}[1]{\left\|{#1}\right\|}

%% file: main.bbl
\begin{thebibliography}{10}

\bibitem{abernethy2009new}
Jacob Abernethy, Francis Bach, Theodoros Evgeniou, and Jean-Philippe Vert.
\newblock A new approach to collaborative filtering: Operator estimation with
  spectral regularization.
\newblock {\em Journal of Machine Learning Research}, 10(Mar):803--826, 2009.

\bibitem{adams2003sobolev}
Robert~A Adams and John~JF Fournier.
\newblock {\em Sobolev spaces}, volume 140.
\newblock Elsevier, 2003.

\bibitem{altae2017low}
Han Altae-Tran, Bharath Ramsundar, Aneesh~S Pappu, and Vijay Pande.
\newblock Low data drug discovery with one-shot learning.
\newblock {\em ACS central science}, 3(4):283--293, 2017.

\bibitem{antoniou2018train}
Antreas Antoniou, Harrison Edwards, and Amos Storkey.
\newblock How to train your maml.
\newblock {\em International conference on learning representations}, 2019.

\bibitem{aronszajn1950theory}
Nachman Aronszajn.
\newblock Theory of reproducing kernels.
\newblock {\em Transactions of the American mathematical society},
  68(3):337--404, 1950.

\bibitem{bakir2007predicting}
G{\"o}khan Bakir, Thomas Hofmann, Bernhard Sch{\"o}lkopf, Alexander~J Smola,
  and Ben Taskar.
\newblock {\em Predicting structured data}.
\newblock MIT press, 2007.

\bibitem{bartlett2006}
Peter~L Bartlett, Michael~I Jordan, and Jon~D McAuliffe.
\newblock Convexity, classification, and risk bounds.
\newblock {\em Journal of the American Statistical Association},
  101(473):138--156, 2006.

\bibitem{berlinet2011reproducing}
Alain Berlinet and Christine Thomas-Agnan.
\newblock {\em Reproducing kernel Hilbert spaces in probability and
  statistics}.
\newblock Springer Science \& Business Media, 2011.

\bibitem{bertinetto2018meta}
Luca Bertinetto, Joao~F Henriques, Philip~HS Torr, and Andrea Vedaldi.
\newblock Meta-learning with differentiable closed-form solvers.
\newblock {\em International conference on learning representations}, 2019.

\bibitem{cai2020weighted}
Diana Cai, Rishit Sheth, Lester Mackey, and Nicolo Fusi.
\newblock Weighted meta-learning.
\newblock {\em arXiv preprint arXiv:2003.09465}, 2020.

\bibitem{caponnetto2007}
Andrea Caponnetto and Ernesto De~Vito.
\newblock Optimal rates for the regularized least-squares algorithm.
\newblock {\em Foundations of Computational Mathematics}, 7(3):331--368, 2007.

\bibitem{chen2018gradnorm}
Zhao Chen, Vijay Badrinarayanan, Chen-Yu Lee, and Andrew Rabinovich.
\newblock Gradnorm: Gradient normalization for adaptive loss balancing in deep
  multitask networks.
\newblock In {\em International Conference on Machine Learning}, pages
  794--803. PMLR, 2018.

\bibitem{ciliberto2019localized}
Carlo Ciliberto, Francis Bach, and Alessandro Rudi.
\newblock Localized structured prediction.
\newblock In {\em Advances in Neural Information Processing Systems}, 2019.

\bibitem{ciliberto2020general}
Carlo Ciliberto, Lorenzo Rosasco, and Alessandro Rudi.
\newblock A general framework for consistent structured prediction with
  implicit loss embeddings.
\newblock {\em Journal of Machine Learning Research}, 21(98):1--67, 2020.

\bibitem{denevi2020advantage}
Giulia Denevi, Massimiliano Pontil, and Carlo Ciliberto.
\newblock The advantage of conditional meta-learning for biased regularization
  and fine-tuning.
\newblock In {\em Advances in Neural Information Processing Systems}, 2020.

\bibitem{fei2006one}
Li~Fei-Fei, Rob Fergus, and Pietro Perona.
\newblock One-shot learning of object categories.
\newblock {\em IEEE transactions on pattern analysis and machine intelligence},
  28(4), 2006.

\bibitem{finn2017model}
Chelsea Finn, Pieter Abbeel, and Sergey Levine.
\newblock Model-agnostic meta-learning for fast adaptation of deep networks.
\newblock In {\em Proceedings of the 34th International Conference on Machine
  Learning-Volume 70}. JMLR. org, 2017.

\bibitem{gretton2012kernel}
Arthur Gretton, Karsten~M Borgwardt, Malte~J Rasch, Bernhard Sch{\"o}lkopf, and
  Alexander Smola.
\newblock A kernel two-sample test.
\newblock {\em Journal of Machine Learning Research}, 13(Mar):723--773, 2012.

\bibitem{ha2016hypernetworks}
David Ha, Andrew Dai, and Quoc~V Le.
\newblock Hypernetworks.
\newblock {\em arXiv preprint arXiv:1609.09106}, 2016.

\bibitem{he2016deep}
Kaiming He, Xiangyu Zhang, Shaoqing Ren, and Jian Sun.
\newblock Deep residual learning for image recognition.
\newblock In {\em Proceedings of the IEEE conference on computer vision and
  pattern recognition}, pages 770--778, 2016.

\bibitem{hochreiter2001learning}
Sepp Hochreiter, A~Steven Younger, and Peter~R Conwell.
\newblock Learning to learn using gradient descent.
\newblock In {\em International Conference on Artificial Neural Networks},
  pages 87--94. Springer, 2001.

\bibitem{jerfel2019reconciling}
Ghassen Jerfel, Erin Grant, Tom Griffiths, and Katherine~A Heller.
\newblock Reconciling meta-learning and continual learning with online mixtures
  of tasks.
\newblock In {\em Advances in Neural Information Processing Systems}, pages
  9119--9130, 2019.

\bibitem{jiang2018learning}
Xiang Jiang, Mohammad Havaei, Farshid Varno, Gabriel Chartrand, Nicolas
  Chapados, and Stan Matwin.
\newblock Learning to learn with conditional class dependencies.
\newblock In {\em International Conference on Learning Representations}, 2018.

\bibitem{kendall2018multi}
Alex Kendall, Yarin Gal, and Roberto Cipolla.
\newblock Multi-task learning using uncertainty to weigh losses for scene
  geometry and semantics.
\newblock In {\em Proceedings of the IEEE conference on computer vision and
  pattern recognition}, pages 7482--7491, 2018.

\bibitem{cifar100}
Alex Krizhevsky, Vinod Nair, and Geoffrey Hinton.
\newblock Cifar 100 dataset.
\newblock \url{https://www.cs.toronto.edu/~kriz/cifar.html}, 2009.

\bibitem{lake2011one}
Brenden Lake, Ruslan Salakhutdinov, Jason Gross, and Joshua Tenenbaum.
\newblock One shot learning of simple visual concepts.
\newblock In {\em Proceedings of the annual meeting of the cognitive science
  society}, volume~33, 2011.

\bibitem{lee2019learning}
Hae~Beom Lee, Hayeon Lee, Donghyun Na, Saehoon Kim, Minseop Park, Eunho Yang,
  and Sung~Ju Hwang.
\newblock Learning to balance: Bayesian meta-learning for imbalanced and
  out-of-distribution tasks.
\newblock {\em International Conference on Learning Representations}, 2020.

\bibitem{lee2019meta}
Kwonjoon Lee, Subhransu Maji, Avinash Ravichandran, and Stefano Soatto.
\newblock Meta-learning with differentiable convex optimization.
\newblock In {\em Proceedings of the IEEE Conference on Computer Vision and
  Pattern Recognition}, pages 10657--10665, 2019.

\bibitem{li2016learning}
Ke~Li and Jitendra Malik.
\newblock Learning to optimize.
\newblock {\em arXiv preprint arXiv:1606.01885}, 2016.

\bibitem{li2017meta}
Zhenguo Li, Fengwei Zhou, Fei Chen, and Hang Li.
\newblock Meta-sgd: Learning to learn quickly for few-shot learning.
\newblock {\em arXiv preprint arXiv:1707.09835}, 2017.

\bibitem{luise2018differential}
Giulia Luise, Alessandro Rudi, Massimiliano Pontil, and Carlo Ciliberto.
\newblock Differential properties of sinkhorn approximation for learning with
  wasserstein distance.
\newblock In {\em Advances in Neural Information Processing Systems}, pages
  5859--5870, 2018.

\bibitem{meanti2020kernel}
Giacomo Meanti, Luigi Carratino, Lorenzo Rosasco, and Alessandro Rudi.
\newblock Kernel methods through the roof: handling billions of points
  efficiently.
\newblock In {\em Advances in Neural Information Processing Systems}, 2020.

\bibitem{mroueh2012multiclass}
Youssef Mroueh, Tomaso Poggio, Lorenzo Rosasco, and Jean-Jeacques Slotine.
\newblock Multiclass learning with simplex coding.
\newblock In {\em Advances in Neural Information Processing Systems}, pages
  2789--2797, 2012.

\bibitem{nichol2018first}
Alex Nichol, Joshua Achiam, and John Schulman.
\newblock On first-order meta-learning algorithms.
\newblock {\em arXiv preprint arXiv:1803.02999}, 2018.

\bibitem{nowozin2011structured}
Sebastian Nowozin, Christoph~H Lampert, et~al.
\newblock Structured learning and prediction in computer vision.
\newblock {\em Foundations and Trends{\textregistered} in Computer Graphics and
  Vision}, 6(3--4):185--365, 2011.

\bibitem{oreshkin2018tadam}
Boris Oreshkin, Pau~Rodr{\'\i}guez L{\'o}pez, and Alexandre Lacoste.
\newblock Tadam: Task dependent adaptive metric for improved few-shot learning.
\newblock In {\em Advances in Neural Information Processing Systems}, pages
  721--731, 2018.

\bibitem{qiao2018few}
Siyuan Qiao, Chenxi Liu, Wei Shen, and Alan~L Yuille.
\newblock Few-shot image recognition by predicting parameters from activations.
\newblock In {\em Proceedings of the IEEE Conference on Computer Vision and
  Pattern Recognition}, pages 7229--7238, 2018.

\bibitem{rajeswaran2019meta}
Aravind Rajeswaran, Chelsea Finn, Sham~M Kakade, and Sergey Levine.
\newblock Meta-learning with implicit gradients.
\newblock In {\em Advances in Neural Information Processing Systems}, 2019.

\bibitem{ravi2016optimization}
Sachin Ravi and Hugo Larochelle.
\newblock Optimization as a model for few-shot learning.
\newblock {\em International conference on learning representations}, 2017.

\bibitem{rodriguez2020embedding}
Pau Rodr{\'\i}guez, Issam Laradji, Alexandre Drouin, and Alexandre Lacoste.
\newblock Embedding propagation: Smoother manifold for few-shot classification.
\newblock {\em arXiv preprint arXiv:2003.04151}, 2020.

\bibitem{rudi2017falkon}
Alessandro Rudi, Luigi Carratino, and Lorenzo Rosasco.
\newblock Falkon: An optimal large scale kernel method.
\newblock In {\em Advances in Neural Information Processing Systems}, pages
  3888--3898, 2017.

\bibitem{rudi2018manifold}
Alessandro Rudi, Carlo Ciliberto, GianMaria Marconi, and Lorenzo Rosasco.
\newblock Manifold structured prediction.
\newblock In {\em Advances in Neural Information Processing Systems}, pages
  5610--5621, 2018.

\bibitem{rusu2018meta}
Andrei~A Rusu, Dushyant Rao, Jakub Sygnowski, Oriol Vinyals, Razvan Pascanu,
  Simon Osindero, and Raia Hadsell.
\newblock Meta-learning with latent embedding optimization.
\newblock {\em International conference on learning representations}, 2019.

\bibitem{santoro2016meta}
Adam Santoro, Sergey Bartunov, Matthew Botvinick, Daan Wierstra, and Timothy
  Lillicrap.
\newblock Meta-learning with memory-augmented neural networks.
\newblock In {\em International conference on machine learning}, 2016.

\bibitem{shalev2014understanding}
Shai Shalev-Shwartz and Shai Ben-David.
\newblock {\em Understanding machine learning: From theory to algorithms}.
\newblock Cambridge university press, 2014.

\bibitem{snell2017prototypical}
Jake Snell, Kevin Swersky, and Richard Zemel.
\newblock Prototypical networks for few-shot learning.
\newblock In {\em Advances in Neural Information Processing Systems}, pages
  4077--4087, 2017.

\bibitem{sriperumbudur2010hilbert}
Bharath~K Sriperumbudur, Arthur Gretton, Kenji Fukumizu, Bernhard
  Sch{\"o}lkopf, and Gert~RG Lanckriet.
\newblock Hilbert space embeddings and metrics on probability measures.
\newblock {\em Journal of Machine Learning Research}, 11(Apr):1517--1561, 2010.

\bibitem{taskar2004max}
Ben Taskar, Carlos Guestrin, and Daphne Koller.
\newblock Max-margin markov networks.
\newblock In {\em Advances in neural information processing systems}, pages
  25--32, 2004.

\bibitem{thrun1996learning}
Sebastian Thrun.
\newblock Is learning the n-th thing any easier than learning the first?
\newblock In {\em Advances in neural information processing systems}, pages
  640--646, 1996.

\bibitem{tsochantaridis2005}
Ioannis Tsochantaridis, Thorsten Joachims, Thomas Hofmann, and Yasemin Altun.
\newblock Large margin methods for structured and interdependent output
  variables.
\newblock In {\em Journal of Machine Learning Research}, pages 1453--1484,
  2005.

\bibitem{vilalta2002perspective}
Ricardo Vilalta and Youssef Drissi.
\newblock A perspective view and survey of meta-learning.
\newblock {\em Artificial intelligence review}, 18(2):77--95, 2002.

\bibitem{vinyals2016matching}
Oriol Vinyals, Charles Blundell, Timothy Lillicrap, Daan Wierstra, et~al.
\newblock Matching networks for one shot learning.
\newblock In {\em Advances in neural information processing systems}, 2016.

\bibitem{vuorio2019multimodal}
Risto Vuorio, Shao-Hua Sun, Hexiang Hu, and Joseph~J Lim.
\newblock Multimodal model-agnostic meta-learning via task-aware modulation.
\newblock In {\em Advances in Neural Information Processing Systems}, pages
  1--12, 2019.

\bibitem{wang2019tafe}
Xin Wang, Fisher Yu, Ruth Wang, Trevor Darrell, and Joseph~E Gonzalez.
\newblock Tafe-net: Task-aware feature embeddings for low shot learning.
\newblock In {\em Proceedings of the IEEE Conference on Computer Vision and
  Pattern Recognition}, pages 1831--1840, 2019.

\bibitem{yao2019hierarchically}
Huaxiu Yao, Ying Wei, Junzhou Huang, and Zhenhui Li.
\newblock Hierarchically structured meta-learning.
\newblock In {\em International Conference on Machine Learning}, pages
  7045--7054, 2019.

\bibitem{yao2007early}
Yuan Yao, Lorenzo Rosasco, and Andrea Caponnetto.
\newblock On early stopping in gradient descent learning.
\newblock {\em Constructive Approximation}, 26(2):289--315, 2007.

\bibitem{zintgraf2018fast}
Luisa~M Zintgraf, Kyriacos Shiarlis, Vitaly Kurin, Katja Hofmann, and Shimon
  Whiteson.
\newblock Fast context adaptation via meta-learning.
\newblock In {\em Proceedings of the 36th International Conference on Machine
  Learning-Volume 70}. JMLR. org, 2019.

\end{thebibliography}
